\documentclass[11pt]{article}
\usepackage{dilab_arxiv}

\usepackage{amsmath,amsfonts,bm}

\def\eqref#1{equation~\ref{#1}}
\def\Eqref#1{Equation~\ref{#1}}

\def\1{\bm{1}}

\DeclareMathAlphabet{\mathsfit}{\encodingdefault}{\sfdefault}{m}{sl}
\SetMathAlphabet{\mathsfit}{bold}{\encodingdefault}{\sfdefault}{bx}{n}

\usepackage{enumitem}
\usepackage{flafter}
\usepackage{placeins}
\usepackage{bookmark}
\apptocmd{\thebibliography}{\raggedright}{}{}

\newtheorem{corollary}{Corollary}

\setlist[itemize]{leftmargin=*,itemsep=3pt,topsep=5pt}

\hypersetup{
  pdftitle={The Router Within: Eliciting Native Skill Routing from a Frozen LLM},
  pdfauthor={Ruishuo Chen, Xun Wang, Yu Chen, Zhuoran Li, and Longbo Huang},
  pdfsubject={Gavel: skill routing from a frozen language model},
  pdfkeywords={large language models, agents, skills, routing, Gavel},
  bookmarksnumbered=true
}

\usepackage{wrapfig}
\newtheorem{theorem}{Theorem}
\usepackage[ruled,vlined,linesnumbered]{algorithm2e}
\usepackage{subcaption}
\usepackage{array}
\usepackage{graphicx}
\usepackage{makecell}

\title{Flowing Faster to Coordinate: One-Step Online Multi-Agent Flow Policies}
\runningtitle{Flowing Faster to Coordinate: One-Step Online Multi-Agent Flow Policies}
\date{arXiv preprint, October 2026}
\paperlogo{\includegraphics[height=1.5cm]{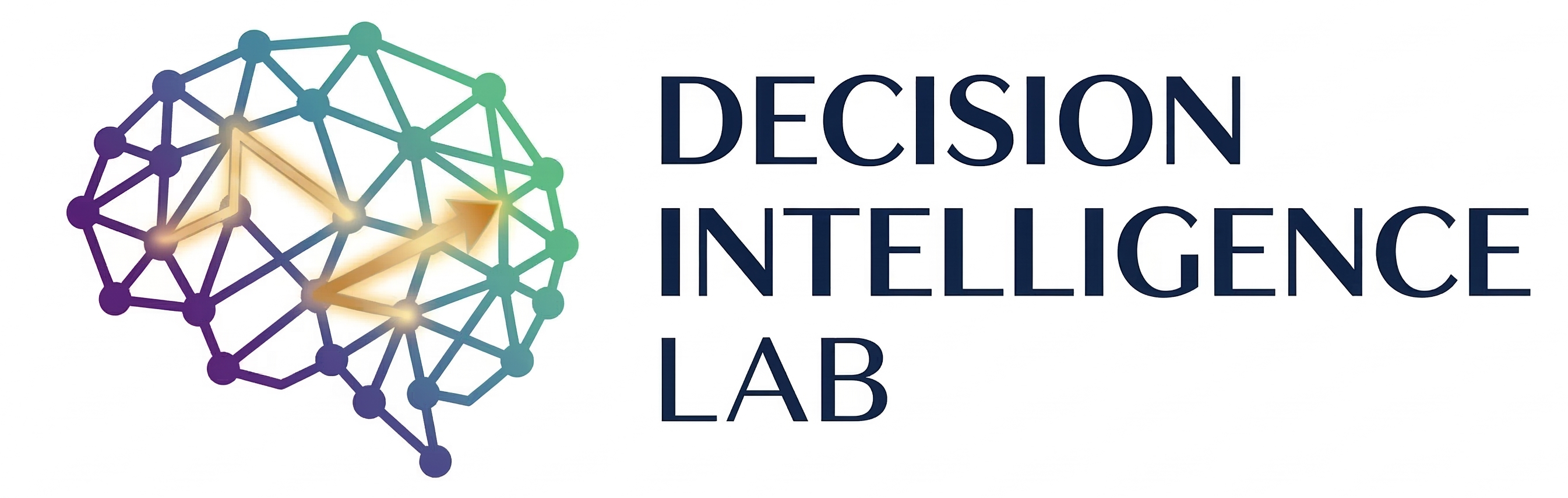}}
\author{
  Zhuoran Li$^{1}$, Yunzhan Li$^{1}$, Xun Wang$^{1}$, Yihan Du$^{2}$, and
  Longbo Huang$^{1\,\text{\faEnvelope}}$
  \\[0.3em]\normalfont
  $^1$Institute for Interdisciplinary Information Sciences, Tsinghua University
  \\
  $^2$Engineering Systems and Design Pillar, Singapore University of Technology and Design
  \\
  \text{\faEnvelope}\ Correspondence: \href{mailto:longbohuang@tsinghua.edu.cn}{\texttt{longbohuang@tsinghua.edu.cn}}
}

\begin{document}
\maketitle
\thispagestyle{fancy}

\begin{abstract}
Multi-agent reinforcement learning (MARL) provides a powerful framework for learning coordinated behaviors through interactions with the environment. Developing MARL policies requires balancing expressive modeling of complex and multimodal action distributions with efficient training and execution. Generative policies, particularly diffusion-based policies, can faithfully capture complex and multimodal behaviors, but costly iterative sampling hinders their scalability in online multi-agent settings. We propose  an \textbf{\underline{O}}nline \textbf{\underline{MA}}RL framework via \textbf{one-step \underline{F}low model} (\textbf{OMAF}) that combines expressive generative policies with efficient one-step action generation. OMAF employs a Transformer-based flow policy to capture complex coordination behaviors, while its approximate path score surrogate provides a principled route to synchronized flow policy optimization. To enable stable and sample-efficient learning, we further develop a joint optimization scheme coupling softmax Q-value estimation with a joint flow policy objective for coordinated policy learning. By eliminating iterative sampling, OMAF dramatically reduces training overhead without sacrificing policy expressiveness. 
Extensive experiments across 10 standard tasks from MPE and MAMuJoCo show that OMAF consistently achieves superior performance, with up to $3.4\times$ higher returns and $10.5\times$ sample efficiency improvement compared with baseline methods.
These results validate the effectiveness of OMAF as an expressive and computationally efficient one-step flow policy paradigm for online MARL.
\end{abstract}

\section{Introduction}
\label{sec:intro}

Multi-agent reinforcement learning (MARL)~\citep{oliehoek2008optimal} harnesses the power of collective intelligence, empowering agents to unlock capabilities beyond those of isolated learners. This collaborative paradigm has driven remarkable advances across diverse domains, from superhuman gameplay~\citep{vinyals2019grandmaster,berner2019dota,ye2020towards} to embodied coordination~\citep{chen2025multi,guo2026opencda}, scientific discovery, e.g., mathematical theorem proving~\citep{su2025multi,openai2026finite}, and LLM-based agentic systems~\citep{zhao2026stronger,yang2026learning}. Realizing such collective capabilities requires policies that are both sample-efficient and expressive enough to capture diverse coordination behaviors~\citep{anonymous2025heterogeneous,li2026omad}.

Generative policies, particularly diffusion- and flow-based policies~\citep{ho2020denoising,lipmanflow}, provide a powerful foundation for expressive decision-making, naturally capturing complex and multimodal action distributions and diverse, coordinated behaviors beyond conventional unimodal policies~\citep{wang2022diffusion,zhang2026sac}. Such advantage has translated into strong performance with diffusion policies~\citep{maefficient,celikdime} in single-agent online RL, but their iterative sampling remain computationally costly. One-step flow policies overcome this bottleneck with single-step action generation while retaining the benefits of generative modeling~\citep{zhan2026mean,dong2026mean}.

Yet, efficient generative policies remain underexplored in multi-agent RL, despite being particularly important in this setting. The exponentially larger joint state-action spaces of multiple interacting agents demand expressive policies while amplifying the cost of iterative generation as the number of agents grows~\citep{lee2026macflow,li2026omad}. Meanwhile, limited interaction data and complex coordination further require high sample efficiency. More critically, few-step flow generation faces a trade-off between efficiency and policy expressiveness, making it challenging in multi-agent settings. Developing expressive generative policies that retain strong coordination performance with efficient one-step generation therefore remains an open challenge for online MARL.

\begin{wrapfigure}{r}{0.38\textwidth}
    \centering
    \includegraphics[width=\linewidth]{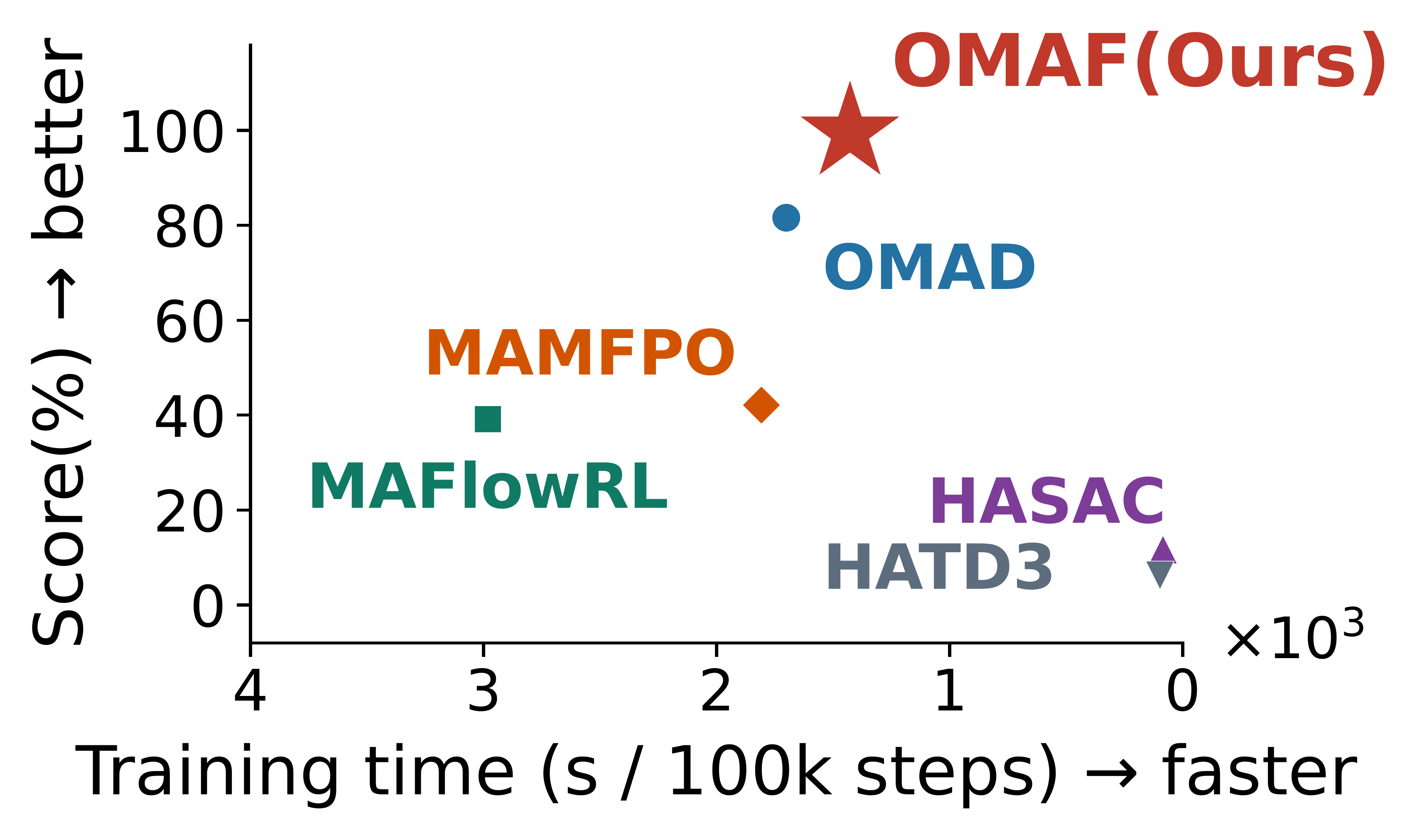}
    \caption{Superior performance and improved training efficiency of OMAF over 10 standard tasks.}
    \label{fig:overview}
\end{wrapfigure}

To tackle these challenges, we propose \textbf{\underline{O}}nline \textbf{\underline{MA}}RL via \textbf{one-step \underline{F}low model} (\textbf{OMAF}), a framework specifically designed for efficient and expressive policy learning in online MARL. OMAF introduces a Transformer-based one-step flow policy for efficient generative policy learning, retaining the expressive coordination while eliminating iterative denoising. We further develop an approximate path score surrogate and synchronized policy optimization with a joint softmax Q-value~\citep{pan2020softmax} for stable and sample-efficient training from scratch. Together, these designs jointly address the efficiency–performance challenge of generative policy learning in online MARL. As illustrated in Figure~\ref{fig:overview}, OMAF simultaneously achieves state-of-the-art performance and training efficiency across multi-agent environments. Our contributions are as follows:

\begin{itemize}
\item We propose \textbf{OMAF}, an efficient one-step flow policy for online MARL under the centralized training and decentralized execution (CTDE) framework. The policy incorporates a Transformer-based velocity network to generate continuous actions through a single denoising step, retaining expressive multi-agent coordination without iterative sampling.

\item We develop a stable joint optimization scheme to effectively train the one-step flow policy. Softmax Q-value estimation provides smooth and reliable value guidance, while synchronized joint flow policy optimization coordinates policy updates across agents and promotes effective multi-agent coordination. Together, these components facilitate stable and sample-efficient learning of expressive one-step flow policies.

\item We conduct extensive experiments across 10 standard online MARL environments from Multi-Agent Particle Environments (MPE) and Multi-Agent MuJoCo (MAMuJoCo). OMAF consistently achieves the best performance, with up to $3.4\times$ higher returns and $10.5\times$ greater sample efficiency than baseline methods, while substantially reducing computational overhead by eliminating iterative denoising.

\end{itemize}
\section{Related Work}
In this section, we discuss representative works on online MARL and flow-based RL policies, and defer a more comprehensive review to Appendix~\ref{appendix:related_works}.

\subsection{Online Multi-Agent Reinforcement Learning}
Online MARL has advanced substantially under centralized training with decentralized execution (CTDE)~\citep{oliehoek2008optimal,matignon2012coordinated}, with representative methods such as MADDPG, MAPPO, VDN, QMIX, and QPLEX~\citep{lowe2017multi,yu2021surprising,sunehag2017value,rashid2018qmix,wangqplex} improving coordination through centralized critics and value factorization. Recent studies further enhance factorization, exploration, robustness, scalability, and heterogeneous-agent coordination~\citep{huang2026pow,liu2025qcofr,qu2026drigm,pan2021regularized,li2023race,zhong2024heterogeneous}, while decentralized methods address scalability and non-stationarity~\citep{tampuu2017multiagent,de2020independent,li2025rac,meulemans2025learning}. 
Recent works introduce generative policies, including  HAQO~\citep{anonymous2025heterogeneous}, OMAD~\citep{li2026omad} and DDPL~\citep{balim2026ddpl}, to improve policy expressiveness and exploration.

\subsection{Flow Policies in Reinforcement Learning}
Flow-based generative policies offer efficient alternatives to diffusion by learning velocity fields that transport noise to actions along straighter trajectories~\citep{lipmanflow,liuflow,albergobuilding,geng2025mean}. In RL, existing works study flow models for value estimation, e.g., flowq~\citep{agrawalla2026floq,agrawalla2026what}, and expressive policies, including ReinFlow~\citep{zhang2025reinflow}, SAC-Flow~\citep{zhang2026sac}, FlowQ~\citep{alles2025flowqenergyguidedflowpolicies}, QIPO~\citep{zhangenergy}, and FQL~\citep{park2025fql}, as well as efficient one-step generation through SSCP~\citep{koirala2025flowbasedsinglestepcompletionefficient}, OFQL~\citep{nguyen2026ofql}, and MVP~\citep{zhan2026mean}. Recent methods further optimize flow policies via policy gradients, including FPO~\citep{mcallister2026flow}, OSMFPG~\citep{wang2026one}, LaMPO~\citep{liu2026latent}, FlowRL~\citep{lyu2026flow}, RFO~\citep{zhong2026reparameterization}, MFPO~\citep{dong2026mean}, and FPMD~\citep{chen2025one}. For MARL, MAC-Flow~\citep{lee2026macflow} and OM$^2$P~\citep{li2026om2p} consider one-step flow policies only in offline learning. We address the missing setting of direct one-step flow optimization in online MARL.

\section{Background}

\subsection{Multi-Agent Reinforcement Learning and Efficient Value Estimation}\label{background:MARL}
We formulate cooperative MARL as a decentralized partially observable MDP (Dec-POMDP), notated as $\langle \mathcal{S},\mathcal{A},\mathcal{P},\mathcal{R},\rho\rangle$~\citep{oliehoek2016concise} with global states, joint actions, transition dynamics, shared rewards, and initial state distribution. The transition function $\mathcal{P}$ specifies the transition dynamics, while each agent $i$ samples $a_t^i\sim\pi_i(\cdot|s_t^i)$ based on $s_t^i$, inducing the factorized joint policy $\boldsymbol{\pi}(a|s)=\prod_{i=1}^{N}\pi_i(a^i|s^i)$. Under the shared reward $r_t=\mathcal{R}(s_t,a_t)$, the objective is to maximize $J(\boldsymbol{\pi})=\mathbb{E}_{\boldsymbol{\pi}}\left[\sum_{t=0}^{\infty}\gamma^t r_t\right],$ where $\gamma\in[0,1)$ is the discount factor. To encourage exploration in online MARL, Maximum-entropy RL~\citep{ziebart2008maximum,haarnoja2018soft} augments the objective with an entropy regularizer to encourage exploration: $J_{\mathrm{ent}}(\boldsymbol{\pi})=\mathbb{E}_{\boldsymbol{\pi}}\left[\sum_{t=0}^{\infty}\gamma^t\left(r_t+\alpha\mathcal{H}(\boldsymbol{\pi}(\cdot|s_t))\right)\right],$ where $\alpha$ controls the exploration-exploitation trade-off. The corresponding soft Bellman residual to optimize Q-function $Q_{\phi}$ is
\begin{equation}
\label{eq:base_q}
    \begin{aligned}
        \mathcal{L}_Q(\phi)=
\mathbb{E}_{(s,a,r,s')\sim\mathcal{D}}
\left[\left(Q_\phi(s,a)-r-\gamma \left(\min\limits_{k=1,2}Q_{\tilde{\phi}_k}(s',a')-\alpha\sum_{i=1}^{N}\log \pi_i(a'_i|s')\right)\right)^2\right],
    \end{aligned}
\end{equation}
and the joint policy $\boldsymbol{\pi}_{\theta}$ is updated via
\begin{equation}
\label{eq:base_policy}
L_{\boldsymbol{\pi}}(\theta)=\mathbb{E}_{s\sim\mathcal{D},\mathbf{a}\sim\pi_{\theta}(\cdot|s)}[\alpha \log \boldsymbol{\pi}_{\theta}(\cdot|s)-Q_{\phi}(s,\mathbf{a})].
\end{equation}
The softmax Q-value~\citep{haarnoja2017reinforcement,pan2020softmax} provides a smooth alternative to the hard maximum by weighting actions according to their exponentiated Q-values, and its estimation via importance sampling is written as
\begin{equation}
\label{eq:softmax}
Q_{\text{soft}}(s',\pi)=\mathbb{E}_{a' \sim \pi}
\left[
\frac{\exp\left(\beta Q_{\tilde{\phi}}(s',a')\right)Q_{\tilde{\phi}}(s',a')}{\pi(a'|s')}
\right]/
\mathbb{E}_{a' \sim \pi}
\left[
\frac{\exp\left(\beta Q_{\tilde{\phi}}(s',a')\right)}{\pi(a'|s')}
\right].
\end{equation}

\subsection{Flow Matching}
Flow matching~\citep{lipmanflow,li2026reverse,li2026a} learns a continuous transport process that maps samples from a tractable source distribution to a desired data distribution. Let $x_0\sim p(x)$ denote a data sample and $x_1\sim\mathcal{N}(0,I_d)$ a source sample. A time-dependent path $x_t$ connects the two distributions over $t\in[0,1]$, with its instantaneous velocity given by $v(t,x_t)=\frac{\mathrm{d}x_t}{\mathrm{d}t}$.
Flow matching learns a parameterized velocity field $v_\theta(t,x_t)$ by minimizing
\begin{equation}
\min_{\theta}\
\mathbb{E}_{x_0,x_1,t}
\left[
\left\Vert
v_\theta(t,x_t)-v(t,x_t)
\right\Vert_2^2
\right],
\nonumber
\end{equation}
where $t\sim U([0,1])$. Once trained, the learned flow generates samples by solving the corresponding ordinary differential equation (ODE) $x_t=x_r+\int_r^t v_\theta(\tau,x_\tau)\,\mathrm{d}\tau$, which can be discretized using the first-order update:
\begin{equation}
    x_{t-\Delta t}=x_t-v_\theta(t,x_t)\Delta t\,.
\end{equation}

\section{Theoretical Bridge to OMAF}

While flow matching enables efficient generative modeling through continuous transport, applying flow-based policies to online MARL still presents several challenges.

First, entropy regularization is hindered by intractable policy likelihoods. In contrast to Gaussian policies, flow-based policies involve densities induced by multi-step generative processes, making likelihood evaluation difficult~\citep{lipmanflow}. This complicates entropy-based exploration in continuous multi-agent action spaces~\citep{zhong2024heterogeneous}.

Second, flow policies face a fundamental conflict between joint expressiveness and decentralized execution under CTDE. Joint flows capture inter-agent dependencies but couple agents, while independent flows preserve decentralization at the cost of essential coordination~\citep{zhu2024madiff}.

Finally, one-step generation exposes an intrinsic mismatch between expressive flow policies and online reward optimization. Aggressive step reduction can degrade action quality due to transport errors~\citep{li2026omad}, while RL must reshape the flow toward high-value joint actions~\citep{pan2021regularized,lirevisiting}. Efficiently reconciling expressive one-step generation with global coordination remains an open challenge in online MARL.

These challenges motivate our design of \underline{O}nline \underline{MA}RL via one-step \underline{F}low model (\textbf{OMAF}). We next provide theoretical guarantees and insights to characterize the proposed one-step flow policy and its optimization, offering a principled understanding of its expressiveness, value estimation stability, and coordination capability.

\subsection{Theoretical Guarantees}
\label{subsec:theoretical_guarantees}

We provide population-level theoretical guarantees for OMAF under practical approximations. We first establish convergence of flow policy optimization, showing that the learned policy approaches the soft-optimal decentralized policy. We then bound the approximation error induced by collapsing the flow trajectory into a single Euler step, providing a theoretical basis for one-step generation. Proofs are deferred to Appendix~\ref{appendix:proofs}.

\begin{theorem}[Approximate convergence to the soft-optimal decentralized policy]
\label{thm:soft_policy_convergence}
Let $\rho(\boldsymbol s)$ be the population observation distribution and let
$\boldsymbol\pi^\star$ be the factorized maximum-entropy optimal policy
induced by $Q^\star$.  Suppose that, uniformly over the flow-policy class,
\begin{equation}
\left|\mathcal L_{\pi}(\boldsymbol\theta)
-\overline{\mathcal L}_{\pi}(\boldsymbol\pi_{\boldsymbol\theta})\right|
\leq \varepsilon_{\mathrm{obj}}
\coloneq \varepsilon_Q+\alpha\varepsilon_{\ell},
\nonumber
\label{eq:objective_approximation_error}
\end{equation}
where $\varepsilon_Q$ bounds the critic error and $\varepsilon_{\ell}$ bounds
the discrepancy between the path score surrogate and the exact joint log-density.
Assume that the policy class has approximation error
\begin{equation}
\inf_{\boldsymbol\theta}
\mathbb E_{\boldsymbol s\sim\rho}
\!\left[
\operatorname{KL}\!\left(
\boldsymbol\pi_{\boldsymbol\theta}
\,\Vert\,\boldsymbol\pi^\star
\right)\right]\leq\varepsilon_{\mathrm{app}}.
\nonumber
\label{eq:flow_policy_approximation_error}
\end{equation}
For an iterate $\boldsymbol\theta_k$, define
$\varepsilon_{\mathrm{opt},k}\coloneq\mathcal L_{\pi}(\boldsymbol\theta_k)-
\inf_{\boldsymbol\theta}\mathcal L_{\pi}(\boldsymbol\theta)$.  Then,
\begin{equation}
\mathbb E_{\boldsymbol s\sim\rho}
\!\left[
\operatorname{KL}\!\left(
\boldsymbol\pi_{\boldsymbol\theta_k}
\,\Vert\,\boldsymbol\pi^\star
\right)\right]
\leq
\varepsilon_{\mathrm{app}}+
\frac{\varepsilon_{\mathrm{opt},k}+2\varepsilon_{\mathrm{obj}}}{\alpha}.
\label{eq:soft_policy_kl_bound}
\end{equation}
Consequently, $\boldsymbol\pi^\star$ is the unique policy-space minimizer of
$\overline{\mathcal L}_{\pi}$ when all three errors vanish.  Moreover, if
$\mathcal L_{\pi}$ is $L$-smooth, satisfies a local
Polyak--\L{}ojasiewicz inequality with constant $\mu>0$, and its unbiased
stochastic gradient has variance at most $\sigma_g^2$, then for
$0<\eta\leq 1/L$,
\begin{equation}
\mathbb E[\varepsilon_{\mathrm{opt},k}]
\leq (1-\eta\mu)^k\varepsilon_{\mathrm{opt},0}
+\frac{L\eta\sigma_g^2}{2\mu}.
\label{eq:sgd_optimization_error}
\end{equation}
\end{theorem}

This theorem establishes that the flow policy converges to the soft-optimal policy up to approximation and optimization errors. The next result bounds the single-step Euler approximation error of the stochastic flow policy, where $z\sim\mathcal N(0,I)$ is the joint base noise and $a_t^\star$ denotes the optimal latent flow characteristic induced by the optimal policy $\boldsymbol{\pi}^{*}$.

\begin{theorem}[One-step near-optimality of the stochastic flow policy]
\label{thm:one_step_optimality}
Let $v^\star$ be the velocity field of the optimal latent flow, with material acceleration
$D_tv^\star\coloneq\partial_tv^\star+v^\star\nabla_a v^\star$. Assume that the target flow has bounded temporal variation and that the learned initial velocity satisfies bounded mean approximation and stochastic errors:
\begin{equation}
\begin{aligned}
\left(\mathbb E_{\boldsymbol s,z}
\left[\Vert D_t v^\star(a_t^\star,t\mid\boldsymbol s)\Vert_2^2\right]\right)^{1/2}
&\leq \kappa,\qquad t\in[0,1],\\
\left(\mathbb E_{\boldsymbol s,z}
\left[\Vert\mu_{\boldsymbol\theta}(z,\boldsymbol s,0)
-v^\star(z,0\mid\boldsymbol s)\Vert_2^2\right]\right)^{1/2}
&\leq \varepsilon_\mu,\\
\left(\mathbb E_{\boldsymbol s,z,\boldsymbol\xi}
\left[\Vert\sigma_{\boldsymbol\theta}(z,\boldsymbol s,0)
\odot\boldsymbol\xi\Vert_2^2\right]\right)^{1/2}
&\leq \varepsilon_\sigma,
\qquad \boldsymbol\xi\sim\mathcal N(0,I).
\end{aligned}
\nonumber
\label{eq:velocity_bound_error}
\end{equation}
Let $\boldsymbol\pi^{(1)}_{\boldsymbol\theta}$ be induced by the stochastic
one-step action
\begin{equation}
\widehat{\boldsymbol a}
=\tanh\!\left(z+\mu_{\boldsymbol\theta}(z,\boldsymbol s,0)
+\sigma_{\boldsymbol\theta}(z,\boldsymbol s,0)\odot\boldsymbol\xi\right).
\label{eq:one_step_stochastic_action}
\nonumber
\end{equation}
Then, with $\Delta_1\coloneq\varepsilon_\mu+\varepsilon_\sigma+\kappa/2$,
\begin{equation}
\left(
\mathbb E_{\boldsymbol s\sim\rho}\!\left[
W_2^2\!\left(
\boldsymbol\pi^{(1)}_{\boldsymbol\theta}(\cdot\mid\boldsymbol s),
\boldsymbol\pi^\star(\cdot\mid\boldsymbol s)
\right)
\right]\right)^{1/2}\leq\Delta_1.
\label{eq:one_step_wasserstein_bound}
\end{equation}
If $Q^\star(\boldsymbol s,\boldsymbol a)$ is $K_Q$-Lipschitz in the joint action and
the discounted observation occupancy of
$\boldsymbol\pi^{(1)}_{\boldsymbol\theta}$ is dominated by $\rho$ with
constant $C_\rho$, then
\begin{equation}
J(\boldsymbol\pi^\star)-J(\boldsymbol\pi^{(1)}_{\boldsymbol\theta})
\leq \frac{C_\rho K_Q}{1-\gamma}\Delta_1.
\label{eq:one_step_performance_bound}
\end{equation}
Thus, exact one-step recovery follows in the special case
$\varepsilon_\mu=\varepsilon_\sigma=\kappa=0$; otherwise, the performance
gap is controlled by the velocity-estimation error, the injected velocity noise, and
the curvature of the optimal flow trajectory.
\end{theorem}

This theorem establishes the near-optimality of one-step flow policies under bounded approximation errors. Together, the two theorems underscore the critical role of accurate value estimation and one-step action generation, motivating OMAF as a principled framework specifically designed to enable efficient one-step flow policy learning.
\section{The Online MARL via one-step Flow model  Method}
\label{sec:transformer_flow_policy}

This section introduces our flow policy algorithm for online multi-agent reinforcement learning, with an overview shown in Figure~\ref{fig:omaf}. To enable efficient one-step flow policy learning in online MARL, OMAF incorporates two key designs: (i) a Transformer-based velocity network that directly parameterizes the flow policy, enabling expressive policy modeling and one-step action generation without iterative sampling or policy distillation; and (ii) a stable joint optimization scheme that combines softmax Q-value estimation with joint policy optimization, providing reliable value guidance while facilitating coordinated policy learning. We detail these components below.

\begin{figure}[htbp]
\begin{center}
\includegraphics[width=\linewidth]{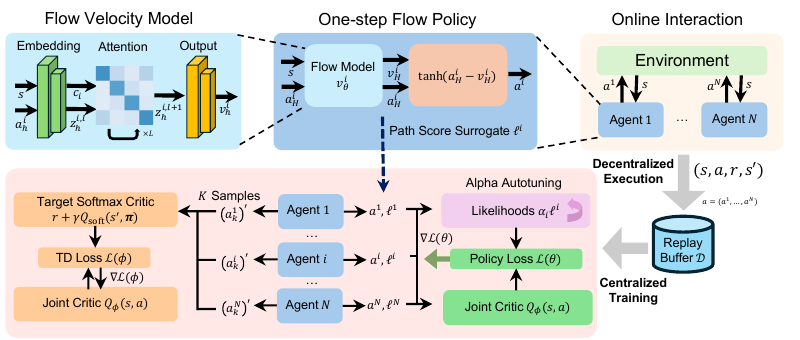}
\end{center}    
\caption{The CTDE framework of OMAF. Each agent generates a one-step flow action $a^i$ and path score surrogate $\ell^i$ through embedding attention (up), while jointly optimizes the flow policies (down) with a shared critic $Q_\phi(s,a)$ and synchronized policy loss $\mathcal{L}(\theta)$.}
\label{fig:omaf}
\end{figure}

OMAF establishes a powerful online flow-based MARL framework that integrates expressive flow policy, efficient one-step generation, and effective multi-agent coordination. 
While existing MARL methods predominantly use Gaussian policies, e.g., HARL~\citep{zhong2024heterogeneous}, OMAF adopts a powerful flow policy to capture expressiveness and complicated action distributions. Compared with diffusion-based methods such as OMAD~\citep{li2026omad}, OMAF replaces costly diffusion-driven iterative denoising with efficient one-step flow generation, achieving expressive policy modeling without sacrificing generation efficiency.

\subsection{Decentralized Flow Policy Formulation}\label{method:flow}
To extend flow-based policies to multi-agent domains, we employ a factorized policy
$\bm{\pi}_{\theta}({a}|s)=\prod_{i=1}^N\pi_{\theta_i}(a^i|s)$
~\citep{zhong2024heterogeneous}, enabling decentralized execution while preserving policy expressiveness. For each agent $i$, we parameterize the velocity field $v_{\theta_i}(a_\tau^i,s,\tau)$ with a Transformer-based network~\citep{vaswani2017attention} (up in Figure~\ref{fig:omaf}). Given $a_1^i\sim\mathcal{N}(0,I)$, the flow trajectory follows
$\frac{\mathrm{d}a_\tau^i}{\mathrm{d}\tau}=v_{\theta_i}(a_\tau^i,s,\tau)$,
where $\tau\in[0,1]$, and is discretized via the first-order Euler update:
\begin{equation}
\label{eq:flow_policy}
a_{\tau-\Delta \tau}^i=a_\tau^i-v_{\theta_i}(a_\tau^i,s,\tau)\Delta \tau.
\end{equation}
Finally, the output action is $a^i=\mathrm{tanh}(a_0^i)$. In OMAF, we use a single Euler step to directly generate the action, i.e.,
$a^i=\mathrm{tanh}(a_1^i-v_{\theta_i}(a_1^i,s,1))$.

Since the exact likelihood is intractable, we approximate it using a path score surrogate~\citep{zhang2026sac} for policy optimization (see Appendix~\ref{appendix:flow_policy} for details). Denote the path density as $p_{\theta_i}(a_{\tau_0^i},\cdots,a_{\tau_H^i})$ for $H$-step discretization, the path score surrogate $\ell^i$ is defined as
\begin{equation}
\ell^i\coloneq p_{\theta_i}(a^i|s)=
\int
\cdots
\int
p_{\theta_i}(a_{\tau_0^i},\cdots,a_{\tau_H^i})
\mathrm{d}a_{\tau_0}^i
\cdots
\mathrm{d}a_{\tau_{H-1}}^i.
\end{equation}

The direct integral is difficult to evaluate, so we employ a tractable approximation to its exact value, with the formulation and analysis detailed in Appendix~\ref{appendix:flow_policy}. This enables efficient policy optimization while retaining expressive flow-based modeling and one-step action generation.

\subsection{Efficient Online Centralized Training of the Flow Policy}
We next formulate the centralized training objective for OMAF, where the flow policy is optimized jointly across agents with entropy regularization to encourage both high-value actions and sufficient exploration (down in Figure~\ref{fig:omaf}). 

\paragraph{Softmax Q-value Estimation.}
The efficient training of the one-step flow policy places a stronger demand on accurate Q-value estimation, as the critic directly guides policy updates. To obtain more reliable Q-value estimates, based on~\Eqref{eq:base_q} and~\Eqref{eq:softmax}, we introduce the centralized critic using the CrossQ~\citep{bhattcrossq} and the softmax Bellman operator~\citep{pan2020softmax}
\begin{equation}
\label{eq:valueloss}
\begin{aligned}
\mathcal{L}_Q(\phi)=
\mathbb{E}
\left[
\left(
r+\gamma Q_{\mathrm{soft}}(s',\boldsymbol{\pi})-Q_\phi(s,{a})
\right)^2
\right],
\end{aligned}
\end{equation}
where $Q_{\mathrm{soft}}$ is approximated using $K$ joint-action samples
$\{{a}'_k\}_{k=1}^{K}$ generated by the flow policies:
\begin{equation}
Q_{\mathrm{soft}}(s',\boldsymbol{\pi})
\approx
\sum_{k=1}^{K}
w_k
\left(
q'_k-e'_k
\right) .
\nonumber
\end{equation}
Specifically, these terms are defined as
\begin{equation}
q'_k=\min_{c\in\{1,2\}}Q_{\phi_c}(s',{a}'_k),
\qquad
e'_k=\sum_{i=1}^{N}\alpha_i{\ell}^{\,i}_k,
\qquad
w_k=
\frac{\exp(\beta q'_k)}
{\sum_{l=1}^{K}\exp(\beta q'_l)}.
\nonumber
\end{equation}
Here $q'_k$ denotes the original Q-value of the $k$-th joint-action sample, $e'_k$ denotes the aggregated path score ${\ell}_k^i$ across agents via the weights $\alpha_i$, and $w_k$ denotes the $\beta$-softmax weight assigned to the $k$-th sample based on its Q-value. Together, softmax Q-value estimation and the path score provide complementary guidance for one-step flow learning: accurate Q-values enable effective policy improvement with few steps, while the path score facilitates efficient optimization and exploration.

\paragraph{Synchronized flow policy optimization.}
To facilitate effective coordination, based on~\Eqref{eq:base_policy}, we synchronously optimize all flow policies under a centralized objective:
\begin{equation}
  \mathcal{L}_{\boldsymbol{\pi}}(\boldsymbol{\theta})=\mathbb{E}_{s\sim\mathcal{D},a_i\sim\boldsymbol{\pi}_{{\theta_i}}(\cdot|s)}
  \left[
    \sum_{i=1}^{N}\alpha_i{\ell}^{i}
    -
    \min_{c\in\{1,2\}}
    Q_{\phi_c}(s,a)
  \right].
  \label{eq:flow_actor_loss}
\end{equation}
Here $a=(a_1,\cdots,a_N)$ denotes the joint action sampled from all flow policies $\pi_{\theta_i}$. Jointly optimizing the policies with the centralized Q-function accounts for inter-agent behaviors, promoting coordinated policy improvement. Meanwhile, tractable path score surrogate enables direct evaluation and approximated maximum-entropy optimization compared with diffusion policies.

\paragraph{Agent-wise temperature auto-tuning.}
To provide fine-grained exploration control for each agent, we adaptively tune an individual entropy coefficient $\alpha_i$. Denote $\overline{\ell}^{i}\coloneq\mathbb{E}[{\ell}^{i}]$ as the expected path score surrogate of agent $i$. The temperature is optimized by
\begin{equation}
  \mathcal{L}(\alpha)
  =
  \frac{1}{N}\sum_{i=1}^{N}
  \alpha_i
  \left(
    \mathcal{H}^{i}_{\mathrm{target}}-\operatorname{sg}(\overline{\ell}^{i})
  \right),
  \label{eq:temperature_loss}
\end{equation}
where the expected path score surrogate is detached from the temperature update. This
adaptive scheme automatically adjusts $\alpha_i$ to maintain the desired
exploration level during training. 

Overall, online MARL faces severe policy instability as the joint state distribution evolves rapidly, making expressive policy modeling and robust value guidance crucial under limited samples. The flow policy provides expressive one-step parameterization, while the path score enables efficient policy optimization and exploration. Building on this foundation, the softmax Q-value objective provides concise and effective guidance against distributional shift, the synchronized policy objective promotes coordinated improvement across agents, and the agent-specific temperature objectives enable fine-grained exploration control. Together, these components form a unified framework for efficient and stable one-step flow policy learning in online MARL.

\subsection{OMAF Algorithm}

\begin{algorithm}[htbp]
\caption{Online Multi-Agent Flow Policy (OMAF)}
\label{alg:omaf}

\textbf{Initialize:} Critic networks $Q_{\phi, 1}$, $Q_{\phi, 2}$, target networks $Q_{\bar{\phi}, 1}, Q_{\bar\phi, 2}$, flow policy $\pi_{\theta_i}$ for agents
$i=1,2,\ldots,N$, temperature $\{\alpha_i\}_{i=1}^{N}$, policy delay $d_l$,
target entropy $\{\mathcal{H}_{\mathrm{target}}^i\}_{i=1}^{N}$,
replay buffer $\mathcal{D}$, and the threshold buffer size
$L_{\mathrm{init}}$.
\textcolor{blue}{// Initialization}\label{algline:initialization}

\For{$m=1$ \KwTo $M$ steps}{
    \If{Buffer length $L_{\mathcal{D}} < L_{\mathrm{init}}$}{
        Sample a trajectory $(s_t,a_t,r_t,s_{t+1})$ using the random policy and insert it into $\mathcal{D}$.\label{algline:samplerandom}
    }
    \Else{
        Sample a trajectory $(s_t,a_t,r_t,s_{t+1})$
     using $\pi_{\theta_i}$ and insert it into
    $\mathcal{D}$.\label{algline:samplepolicy}
    }
    \textcolor{blue}{// Trajectory Generation}
    
        Sample a mini-batch $B=\{(s,a,r,s')\}$ from $\mathcal{D}$.
        \textcolor{blue}{// Sampling}\label{algline:sampleminibatch}
        
        Sample $K$ candidates
        $a'_k=((a^1_k)',\ldots,(a^N_k)')$ using $\pi_{\theta_i}$,
        calculate $\mathcal{L}(\phi)$ using~\Eqref{eq:valueloss} and optimize $\phi$.
        \textcolor{blue}{// Critic Optimization}\label{algline:criticopt}

        Update target networks via
        $\bar{\phi}_i \leftarrow
        \rho\bar{\phi}_i + (1-\rho)\phi_i$.
        \textcolor{blue}{// Target Network Optimization}\label{algline:targetcritic}

        \If{$m \bmod d_l = 0$}{
            Sample $a=(a^1,\ldots,a^N)$ using $\pi_{\theta_i}$, and calculate $(\ell_{\pi_{\theta_1}},\ldots,\ell_{\pi_{\theta_N}})$ simultaneously,
            calculate $\mathcal{L}(\theta)$ using~\Eqref{eq:flow_actor_loss} and jointly
            optimize $\{\theta_i\}_{i=1}^{N}$.
            \textcolor{blue}{// Synchronized Flow Policy Optimization}\label{algline:policyopt}
            
            Optimize the temperature
            $\alpha_i$ via minimizing~\Eqref{eq:temperature_loss}.
            \textcolor{blue}{// Temperature Optimization}\label{algline:temperatureopt}
            }
    
}
\end{algorithm}

The training procedure of the proposed Online Multi-Agent Flow Policy (OMAF) is summarized in Algorithm~\ref{alg:omaf}. Line~\ref{algline:initialization} initializes the model parameters and replay buffer. Lines~\ref{algline:samplerandom}-~\ref{algline:samplepolicy} collect trajectories using random policies at the beginning and the learned flow policies thereafter. Once the replay buffer is sufficiently populated, Line~\ref{algline:criticopt} updates the centralized critics using $K$ joint actions sampled from the current flow policies and the softmax Q-value objective, followed by a soft update of the target networks in Line~\ref{algline:targetcritic}. Every $d_l$ steps, Line~\ref{algline:policyopt} jointly optimizes all flow policies using the centralized Q-function, while Line~\ref{algline:temperatureopt} tunes the agent-wise entropy temperatures.

Overall, OMAF establishes a pioneering flow-based approach to online MARL that jointly addresses the expressiveness-efficiency trade-off in generative policy learning. Its one-step flow policy models complex action distributions without the iterative sampling required by diffusion policies, while the path score surrogate enables effective maximum-entropy optimization and captures complex multi-agent coordination. A detailed discussion of the baseline algorithms is provided in Appendix~\ref{appendix:discussion}. We next conduct extensive experiments to evaluate whether these designs translate into improved policy performance and training efficiency across a diverse set of multi-agent tasks.

\section{Experiments}

We evaluate OMAF on diverse continuous-control multi-agent benchmarks and compare it with state-of-the-art baselines in terms of performance and efficiency. We further visualize the learned policies to examine their expressiveness and coordination capability from different perspectives. Finally, we conduct ablation studies to evaluate the contribution of each key component of OMAF.

\subsection{Experimental Setup}
\textbf{Enviroments.} We benchmark OMAF on two widely adopted continuous-control MARL suites: Multi-Agent Particle Environments (MPE)~\citep{lowe2017multi} and Multi-Agent MuJoCo (MAMuJoCo)~\citep{peng2021facmac}. We use Cooperative Navigation in MPE, and HalfCheetah, Ant, Walker, Hopper, and Swimmer in MAMuJoCo, covering particle coordination and high-dimensional locomotion tasks. Detailed environment configurations are provided in Appendix~\ref{appendix:env}.

\textbf{Baselines.}
We consider four representative online MARL baselines for comparison.
HATD3~\citep{zhong2024heterogeneous} and HASAC~\citep{liumaximum}
are two widely used actor-critic methods for online MARL. OMAD~\citep{li2026omad} represents the state-of-the-art multi-agent diffusion policy, while
MAFlowRL~\citep{lyu2026flow} and MAMFPO~\citep{dong2026mean} extend
flow-based RL methods to the multi-agent setting. All methods are trained with $3$ random seeds, and
we report the mean performance and standard deviation over $10$ evaluation
episodes. Implementation details, including hyperparameters and network
architectures, are provided in Appendix~\ref{appendix:parameters}.
\subsection{Experiment Results}
\begin{figure}[ht]
    \begin{center}
    \includegraphics[width=\textwidth]{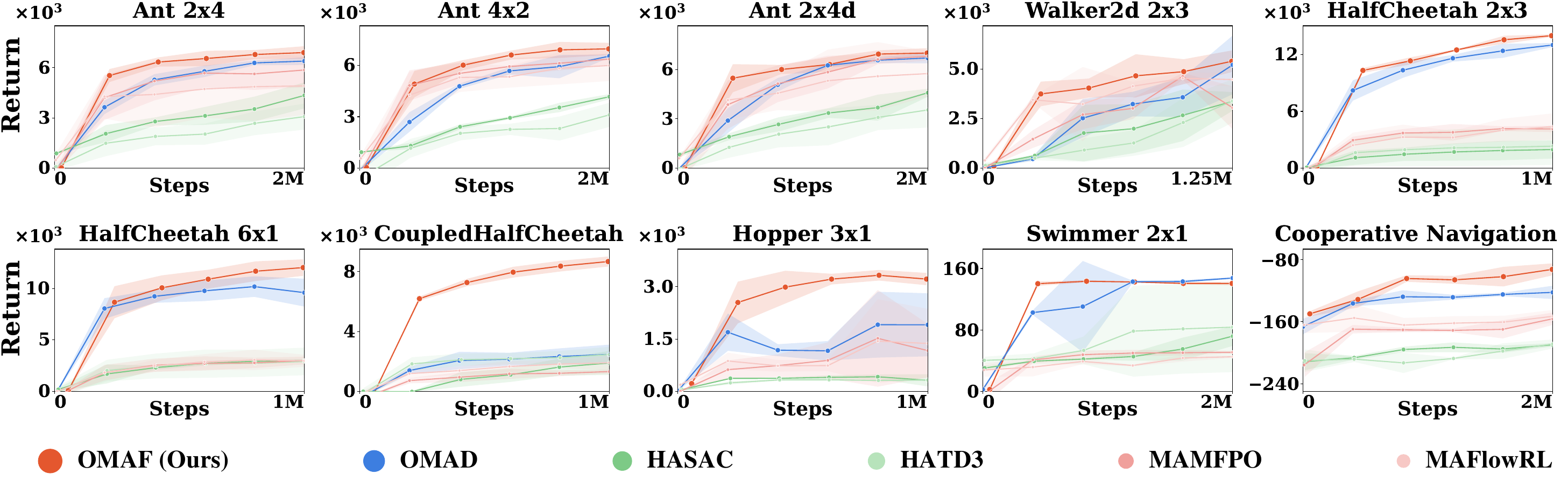}
    \caption{Performance comparison of OMAF against representative online MARL methods HATD3 and HASAC, diffusion-based OMAD, and extensions of flow-based methods MAFlowRL and MAMFPO on the MPE and MAMuJoCo benchmarks. Curves show average episode return over training steps across 3 random seeds, with shaded regions denoting one standard deviation. OMAF achieves superior performance, validating one-step flow policies for online MARL.}
    \label{fig:results}
    \end{center}
\end{figure}

The overall performance of OMAF is presented in Figure~\ref{fig:results}. We evaluate all methods under a limited training budget of fewer than $2\times10^6$ environment interaction steps to assess both performance and sample efficiency. Across all $10$ MPE and MAMuJoCo tasks, OMAF consistently achieves superior performance compared with the baseline algorithms. This advantage holds for both the MPE tasks with complex coordination and the higher-dimensional MAMuJoCo tasks, demonstrating that OMAF can achieve strong performance with substantially fewer environment interactions. Besides, OMAF achieves up to a $10.5\times$ improvement in sample efficiency, reaching baseline-level performance with substantially fewer environment interactions.

\begin{wrapfigure}{r}{0.38\textwidth}
    \centering
    \includegraphics[width=\linewidth]{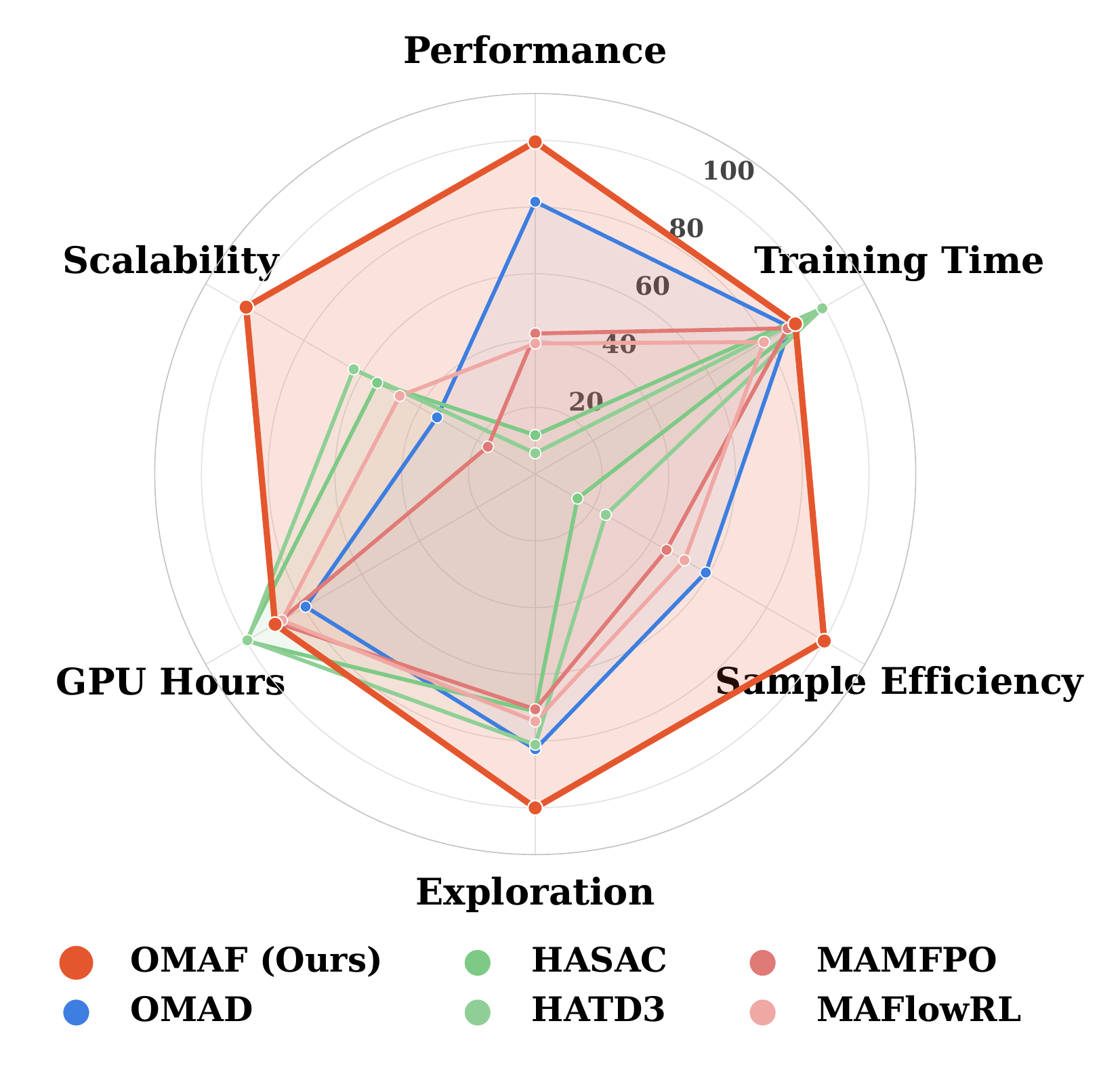}
    \caption{OMAF achieves superior performance across different metrics shown in the radar chart.}
    \label{fig:radar}
\end{wrapfigure}

Compared with Gaussian-policy methods, including HASAC and HATD3~\citep{liumaximum,zhong2024heterogeneous}, OMAF achieves superior performance, demonstrating the advantage of expressive flow-based policies. In contrast to diffusion-based OMAD~\citep{li2026omad}, OMAF achieves better performance without iterative denoising, balancing policy expressiveness and generation efficiency. Against representative multi-agent flow-policy methods, MAFlowRL~\citep{lyu2026flow} and MAMFPO~\citep{dong2026mean}, OMAF consistently performs better, highlighting the effectiveness of our specific flow-policy design and synchronized value and policy optimization. 

As shown in Figure~\ref{fig:radar} (details refer to Appendix~\ref{appendix:exp}), OMAF achieves superior and well-rounded performance across multiple metrics, outperforming representative baselines algorithms. OMAF consistently performs well across task performance, training and sample efficiency, exploration, and scalability, achieving stronger policies with substantially lower computational resource requirements than diffusion-based OMAD~\citep{li2026omad}. Such computational efficiency is crucial for practical and scalable online MARL, where training and inference cost can become a major bottleneck.

\subsection{Ablation Study}
\begin{figure}[ht]
    \centering
    \begin{subfigure}{0.31\textwidth}
        \centering
        \includegraphics[width=\linewidth]{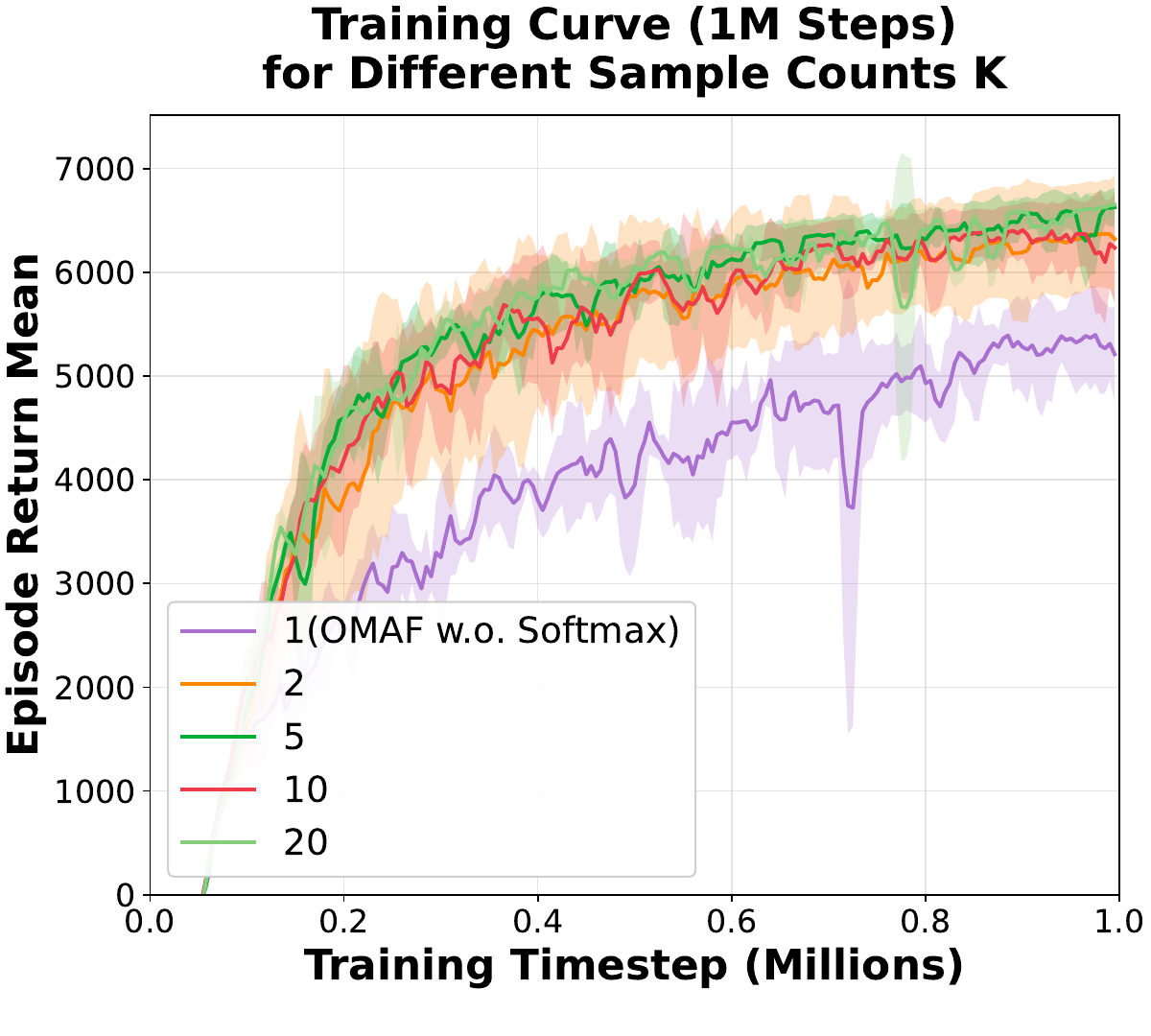}
        \caption{Sample count $K$}
        \label{fig:k}
    \end{subfigure}
    \hfill
    \begin{subfigure}{0.31\textwidth}
        \centering
        \includegraphics[width=\linewidth]{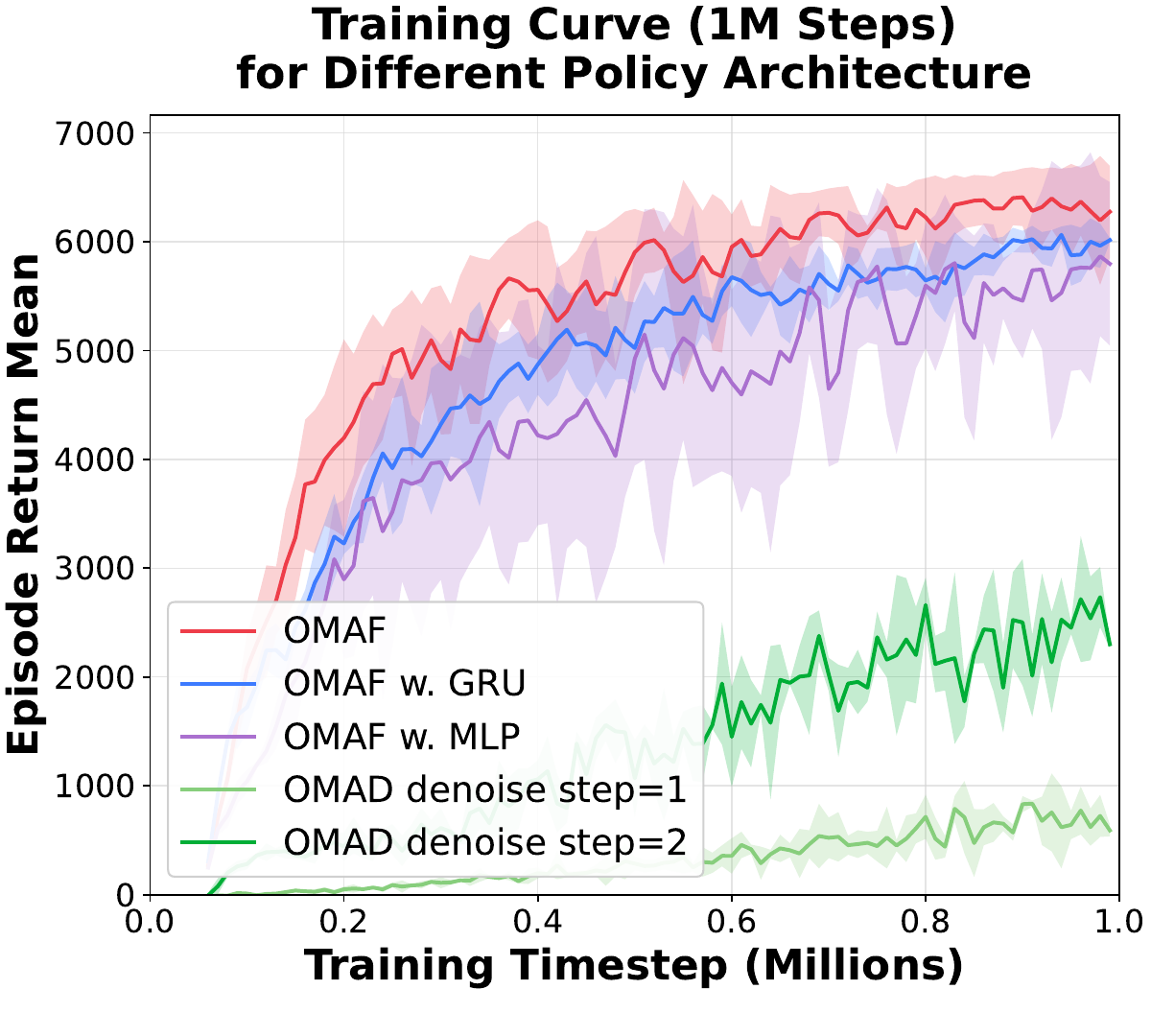}
        \caption{Policy architecture}
        \label{fig:arch}
    \end{subfigure}
    \hfill
    \begin{subfigure}{0.31\textwidth}
        \centering
        \includegraphics[width=\linewidth]{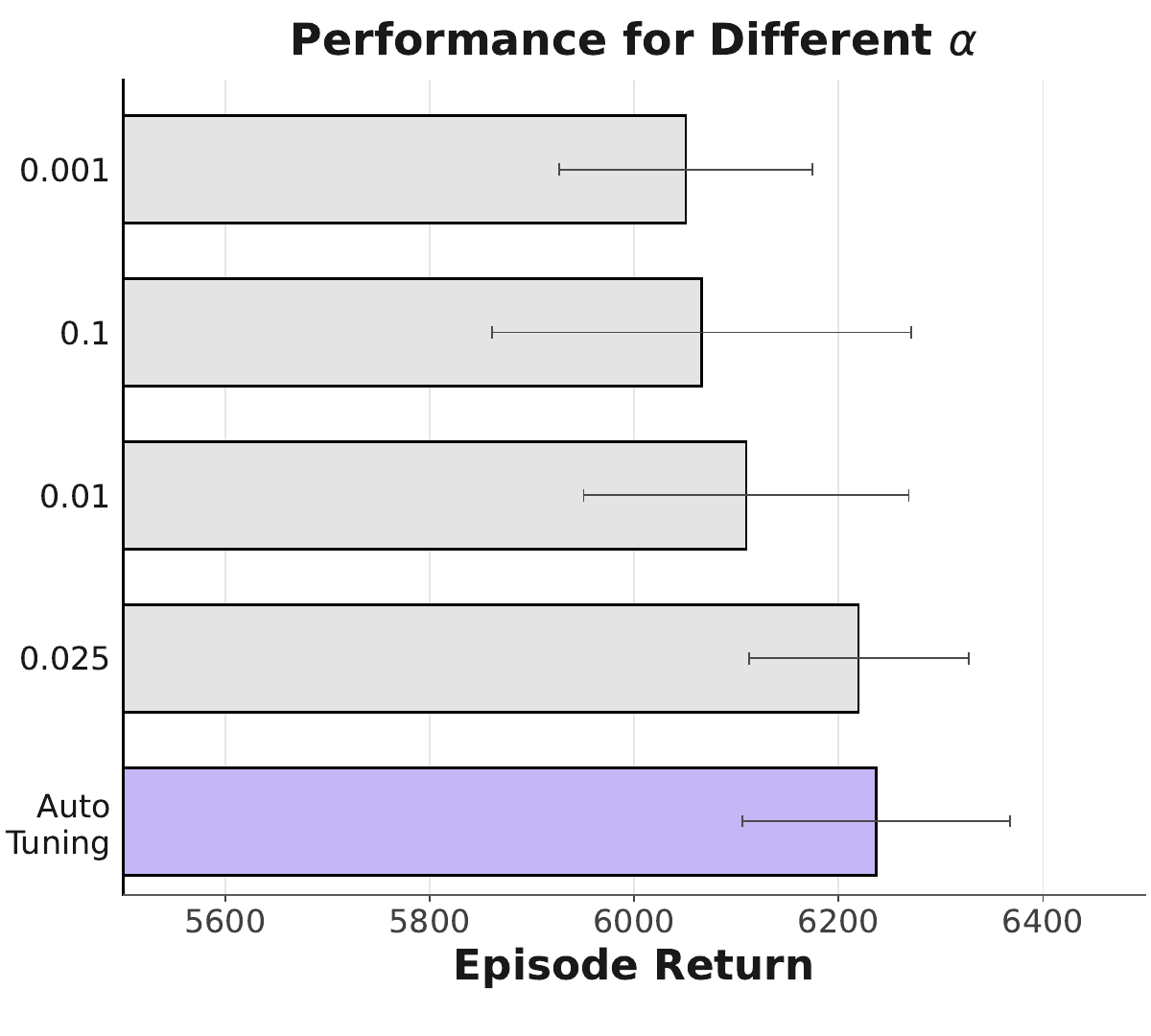}
        \caption{Coefficient term $\alpha_i$}
        \label{fig:alpha}
    \end{subfigure}
    \caption{Ablation studies of OMAF on the Ant $2\times4$ task.}
    \label{fig:ablations}
\end{figure}
To evaluate the individual contributions of key components in OMAF, we conduct comprehensive ablation studies on the Ant $2 \times 4$ task (Figure~\ref{fig:ablations}) about the value estimation, policy architecture and coefficient term sensitivity. More results are shown in Appendix~\ref{appendix:exp}.   

\textbf{Impact of Sample Count $K$ in Softmax Q (Figure~\ref{fig:k}).} Setting $K=1$ (equivalent to OMAF without Softmax Q) leads to a substantial performance degradation, validating the necessity of softmax target value estimation. Selecting an appropriate sample count $K$ significantly stabilizes learning and boosts overall return. However, an excessively large $K$ introduces redundant action sampling, which slightly reduces training efficiency.

\textbf{Policy Architecture and One-Step Efficiency Figure~\ref{fig:arch}).} OMAF utilizing a Transformer backbone consistently outperforms variants based on GRU and MLP, demonstrating that architectural choice plays a critical role in final performance. Furthermore, when reducing the denoising steps of OMAD (denoise step 1 or 2), its performance collapses significantly. This highlights OMAF's superior capability to achieve highly efficient, one-step training without compromising policy quality.   

\textbf{Temperature Coefficient $\alpha$ Auto-Tuning (Figure~\ref{fig:alpha})}: Dynamically auto-tuning $\alpha$ yields higher returns than fixed coefficient settings (e.g., $0.001$, $0.01$, $0.025$, and $0.1$). This confirms that the auto-tuning scheme serves as an effective mechanism to dynamically govern policy exploration.

\section{Conclusion}
We present OMAF, an efficient one-step flow policy framework for online multi-agent reinforcement learning. By integrating expressive flow-based policy modeling with synchronized joint policy optimization and softmax Q-value estimation, OMAF enables efficient action generation while effectively capturing multi-agent coordination. Experiments on MPE and MAMuJoCo demonstrate that OMAF achieves superior performance against strong online diffusion-based and flow-based baselines, while requiring only a single flow step for action generation. These results highlight the potential of one-step flow policies as an efficient and expressive approach for online MARL.
\clearpage
\FloatBarrier

\begingroup
\small
\urlstyle{same}
\bibliographystyle{dilab_ref}
\bibliography{references}
\endgroup

\makeappendixtoc
\appendix
\section{Detailed Related Works}\label{appendix:related_works}
Here, we review related works on online MARL and diffusion- and flow-based policy optimization, covering both single- and multi-agent RL settings. We further discuss the challenges of one-step flow policy optimization in online MARL and compare OMAF with existing approaches.

\subsection{Online Multi-Agent Reinforcement Learning}
Online multi-agent reinforcement learning (MARL) has achieved substantial progress in recent years. Under the centralized training with decentralized execution (CTDE) paradigm~\citep{oliehoek2008optimal,matignon2012coordinated}, a series of influential algorithms, including MADDPG~\citep{lowe2017multi}, MAPPO~\citep{yu2021surprising}, VDN~\citep{sunehag2017value}, COMA~\citep{foerster2018counterfactual}, QTRAN~\citep{son2019qtran}, QMIX~\citep{rashid2018qmix}, and QPLEX~\citep{wangqplex}, have significantly improved coordination and stability in cooperative tasks. In parallel, fully decentralized approaches such as IQL~\citep{tampuu2017multiagent}, IPPO~\citep{de2020independent}, and MATD3~\citep{ackermann2019reducing} achieve competitive performance without centralized critics, underscoring the scalability of online MARL; recent studies further tackle non-stationarity and co-adaptation in fully decentralized settings via return-aware context modeling~\citep{li2025rac} and learning-aware policy gradients~\citep{meulemans2025learning}.

Building on these foundations, recent centralized methods further improve online MARL from several perspectives, including value factorization, exploration, robustness, and scalability. Value decomposition and the IGM principle have been enhanced through more expressive factorization~\citep{huang2026pow,liu2025qcofr,qu2026drigm}, while centralized policy optimization and exploration are improved through joint-policy alignment and adaptive context modeling~\citep{li2026magpo,duan2025adaptive}. Other works study exploration, robustness, and scalability through regularization, curriculum learning, communication constraints, and adversarial training~\citep{pan2021regularized,li2023race,lirevisiting,jungagent,chen2025novelty,kontogiannisenhancing,zhao2025curriculum,zhao2024optimistic,yangmulti,lee2026ibal,jiang2026bhea}. Heterogeneity-aware methods, including HARL and its variants~\citep{zhong2024heterogeneous,liumaximum} and HQP~\citep{anonymous2025heterogeneous}, further extend coordination to heterogeneous agents.

Despite these advances, most online MARL methods use unimodal policies, limiting their ability to represent multimodal coordination strategies. Recent works introduce diffusion policies into online MARL~\citep{li2026omad,balim2026ddpl,anonymous2025heterogeneous}, but iterative denoising incurs substantial sampling and training overhead. Our approach instead employs one-step flow policies, retaining the expressiveness of generative policies while enabling efficient online MARL.

\subsection{Flow Policy in Offline Reinforcement Learning}

Flow-based policies provide an efficient module to achieve iterative denoising by learning velocity fields that enable direct or few-step action generation~\citep{lipmanflow,liuflow,albergobuilding,lipman2024flow,li2026a}. In offline RL, representative approaches such as FQL~\citep{park2025fql}, FlowQ~\citep{alles2025flowqenergyguidedflowpolicies}, QIPO~\citep{zhangenergy}, SSCP~\citep{koirala2025flowbasedsinglestepcompletionefficient}, GFP~\citep{nguimatsia2026gfp}, FLP~\citep{xia2026flp} and OFQL~\citep{nguyen2026ofql} explore flow-based policy learning through one-step generation, energy-guided training, or direct velocity-field learning. Subsequent works further improve flow policies through value-aware generation, behavior regularization, and efficient distillation~\citep{nguimatsia2026gfp,zhang2026goldenstart,xia2026flp,duan2025racd,li2025qchunking}.

These developments naturally extend to multi-agent settings, where expressive generative policies can model diverse joint behaviors and coordination patterns. In offline MARL, early flow-based MARL methods include MAC-Flow~\citep{lee2026macflow} and OM$^2$P~\citep{li2026om2p}, both of which operate offline and rely on multi-stage training or distillation. To the best of our knowledge, no prior work directly learns flow policies within online MARL, which is the setting addressed by our work.

\subsection{Flow Policy in Online RL}

Flow policies offer an expressive and efficient parameterization for reinforcement learning~\citep{nahim2026a}, with representative methods including ReinFlow~\citep{zhang2025reinflow}, SAC-Flow~\citep{zhang2026sac}, GS-Flow~\citep{zhang2026goldenstart}, and MVP~\citep{zhan2026mean}. Recent studies further enable direct policy-gradient optimization of flow policies, including FPO~\citep{mcallister2026flow}, OSMFPG~\citep{wang2026one}, LaMPO~\citep{liu2026latent}, FlowRL~\citep{lyu2026flow}, RFO~\citep{zhong2026reparameterization}, MFPO~\citep{dong2026mean}, RFM~\citep{li2026a}, and FPMD~\citep{chen2025one}. Despite these advances, existing flow policy optimization methods are predominantly developed for single-agent RL and do not address the challenges of one-step flow optimization in online MARL. Our work extends flow-based policy optimization to this setting by enabling direct one-step flow policy optimization under CTDE.

\section{Proofs of Theoretical Results}
\label{appendix:proofs}

\begin{proof}[Proof of Theorem~\ref{thm:soft_policy_convergence}]
Write $F(\boldsymbol\theta)\coloneq\mathcal L_{\pi}(\boldsymbol\theta)$ and
$\overline F(\boldsymbol\pi)\coloneq\overline{\mathcal L}_{\pi}(\boldsymbol\pi)$.

\textbf{Step 1: variational representation.}
For a fixed joint observation, the soft-optimal policy obeys
\begin{equation}
\boldsymbol\pi^\star(\boldsymbol a\mid\boldsymbol s)
=\frac{\exp(Q^\star(s,\boldsymbol a)/\alpha)}{Z(\boldsymbol s)},
\qquad
Z(\boldsymbol s)=\int\exp(Q^\star(s,\boldsymbol a)/\alpha)
\,\mathrm d\boldsymbol a.
\end{equation}
Substitution of
$\alpha\log\boldsymbol\pi^\star=Q^\star-\alpha\log Z$ into the ideal actor
objective gives the exact variational identity
\begin{equation}
\overline F(\boldsymbol\pi)-\overline F(\boldsymbol\pi^\star)
=\alpha\,\mathbb E_{\boldsymbol s\sim\rho}
\!\left[
\operatorname{KL}\!\left(
\boldsymbol\pi(\cdot\mid\boldsymbol s)
\,\Vert\,\boldsymbol\pi^\star(\cdot\mid\boldsymbol s)
\right)\right].
\label{eq:ideal_soft_policy_identity}
\end{equation}
The normalizing term $-\alpha\mathbb E_{\rho}[\log Z(\boldsymbol s)]$ is
independent of the candidate policy and cancels in this difference.  Thus the
ideal actor objective is an information projection onto the soft-optimal
policy within the decentralized factorized policy class.

\textbf{Step 2: objective perturbation.}
Let $\widehat Q$ and $\widehat\ell_{\boldsymbol\theta}$ denote the critic and
joint log likelihood used by the implementation.  Before taking absolute values,
their cumulative deviation from the ideal objective is
\begin{equation}
F(\boldsymbol\theta)-\overline F(\boldsymbol\pi_{\boldsymbol\theta})
=\mathbb E_{\rho,\boldsymbol\pi_{\boldsymbol\theta}}\!\left[
\alpha\bigl(\widehat\ell_{\boldsymbol\theta}
-\log\boldsymbol\pi_{\boldsymbol\theta}\bigr)
-\bigl(\widehat Q-Q^\star\bigr)\right].
\label{eq:actor_objective_perturbation}
\end{equation}
Hence uniform expected errors $\varepsilon_\ell$ and $\varepsilon_Q$ imply~\eqref{eq:objective_approximation_error} by the triangle inequality.
In particular, the theorem does not identify the practical log likelihood with an
exact likelihood; its discrepancy is carried explicitly by
$\varepsilon_\ell$.

\textbf{Step 3: transfer from objective error to policy error.}
Let $\boldsymbol\theta_{\mathrm{app}}$ be an arbitrarily accurate minimizer
of the left-hand side of~\eqref{eq:flow_policy_approximation_error}.
Applying~\eqref{eq:objective_approximation_error} at
$\boldsymbol\theta_k$ and $\boldsymbol\theta_{\mathrm{app}}$ yields the
following four-step transfer:
\begin{equation}
\begin{aligned}
\overline F(\boldsymbol\pi_{\boldsymbol\theta_k})
&\leq F(\boldsymbol\theta_k)+\varepsilon_{\mathrm{obj}} \nonumber\\
&=\inf_{\boldsymbol\theta}F(\boldsymbol\theta)
+\varepsilon_{\mathrm{opt},k}+\varepsilon_{\mathrm{obj}} \nonumber\\
&\leq F(\boldsymbol\theta_{\mathrm{app}})
+\varepsilon_{\mathrm{opt},k}+\varepsilon_{\mathrm{obj}} \nonumber\\
&\leq\overline F(\boldsymbol\pi_{\boldsymbol\theta_{\mathrm{app}}})
+\varepsilon_{\mathrm{opt},k}+2\varepsilon_{\mathrm{obj}}.
\end{aligned}
\label{eq:objective_error_decomposition}
\end{equation}
Subtracting $\overline F(\boldsymbol\pi^\star)$, then applying the
approximation assumption and~\eqref{eq:ideal_soft_policy_identity},
gives
\begin{equation}
\alpha\,\mathbb E_{\rho}\!\left[
\operatorname{KL}\!\left(
\boldsymbol\pi_{\boldsymbol\theta_k}\Vert\boldsymbol\pi^\star\right)\right]
\leq\varepsilon_{\mathrm{opt},k}
+2\varepsilon_{\mathrm{obj}}+\alpha\varepsilon_{\mathrm{app}}.
\label{eq:kl_error_decomposition}
\end{equation}
Dividing by $\alpha$ proves~\eqref{eq:soft_policy_kl_bound}.  In the
realizable and exact-objective limit, the KL divergence vanishes only if the
two policies agree $\rho$-almost surely, proving uniqueness.

\textbf{Step 4: stochastic optimization recursion.}
Let $g_k$ be the mini-batch gradient, satisfying
$\mathbb E[g_k\mid\boldsymbol\theta_k]=\nabla F(\boldsymbol\theta_k)$ and
$\mathbb E[\|g_k-\nabla F(\boldsymbol\theta_k)\|_2^2
\mid\boldsymbol\theta_k]\leq\sigma_g^2$.  From $L$-smoothness and the update
$\boldsymbol\theta_{k+1}=\boldsymbol\theta_k-\eta g_k$,
\begin{align}
\mathbb E[F(\boldsymbol\theta_{k+1})\mid\boldsymbol\theta_k]
&\leq F(\boldsymbol\theta_k)-\eta\|\nabla F(\boldsymbol\theta_k)\|_2^2
+\frac{L\eta^2}{2}
\mathbb E[\|g_k\|_2^2\mid\boldsymbol\theta_k]\nonumber\\
&\leq F(\boldsymbol\theta_k)
-\eta\left(1-\frac{L\eta}{2}\right)
\|\nabla F(\boldsymbol\theta_k)\|_2^2
+\frac{L\eta^2\sigma_g^2}{2}.
\label{eq:stochastic_smoothness_descent}
\end{align}
Because $\eta\leq1/L$, the coefficient of the gradient norm is at least
$\eta/2$.  The PL inequality
$\|\nabla F(\boldsymbol\theta_k)\|_2^2\geq
2\mu\varepsilon_{\mathrm{opt},k}$ consequently gives
\begin{equation}
\mathbb E[\varepsilon_{\mathrm{opt},k+1}\mid\boldsymbol\theta_k]
\leq(1-\eta\mu)\varepsilon_{\mathrm{opt},k}
+\frac{L\eta^2\sigma_g^2}{2}.
\label{eq:conditional_pl_recursion}
\end{equation}
Taking total expectations and unrolling the affine recursion,
\begin{align}
\mathbb E[\varepsilon_{\mathrm{opt},k}]
&\leq(1-\eta\mu)^k\varepsilon_{\mathrm{opt},0}
+\frac{L\eta^2\sigma_g^2}{2}
\sum_{j=0}^{k-1}(1-\eta\mu)^j \nonumber\\
&\leq(1-\eta\mu)^k\varepsilon_{\mathrm{opt},0}
+\frac{L\eta\sigma_g^2}{2\mu},
\end{align}
which establishes~\eqref{eq:sgd_optimization_error}.
\end{proof}

\begin{proof}[Proof of Theorem~\ref{thm:one_step_optimality}] We execute the proof following these steps.

\textbf{Step 1: endpoint error of a one-step discretization.}
Let $\widehat a_1$ denote the latent state before action squashing in~\eqref{eq:one_step_stochastic_action}.  Along the optimal
characteristic $\dot a_1^\star=v^\star(a_t^\star,t\mid\boldsymbol s)$,
\begin{equation}
a_1^\star=z+\int_0^1v^\star(a_t^\star,t\mid\boldsymbol s)\,\mathrm dt.
\end{equation}
Adding and subtracting $v^\star(z,0\mid\boldsymbol s)$ therefore gives
\begin{align}
\widehat a_1-a_1^\star
={}&\mu_{\boldsymbol\theta}(z,\boldsymbol s,0)
-v^\star(z,0\mid\boldsymbol s)
+\sigma_{\boldsymbol\theta}(z,\boldsymbol s,0)\odot\boldsymbol\xi \nonumber\\
&-\int_0^1\!\left[
v^\star(a_t^\star,t\mid\boldsymbol s)
-v^\star(z,0\mid\boldsymbol s)\right]\mathrm dt.
\label{eq:one_step_latent_error}
\end{align}
The fundamental theorem of calculus along this characteristic gives the
non-autonomous flow identity
\begin{equation}
v^\star(a_t^\star,t\mid\boldsymbol s)-v^\star(z,0\mid\boldsymbol s)
=\int_0^tD_sv^\star(a_s^\star,s\mid\boldsymbol s)\,\mathrm ds.
\label{eq:characteristic_velocity_identity}
\end{equation}
Thus, the third term in~\eqref{eq:one_step_latent_error} is exactly the
local truncation error caused by replacing the full characteristic integral
with its initial velocity. Applying Minkowski's integral inequality twice yields
\begin{align}
\left(\mathbb E_{\boldsymbol s,z,\boldsymbol\xi}
\left[\|\widehat a_1-a_1^\star\|_2^2\right]\right)^{1/2}
&\leq\varepsilon_\mu+\varepsilon_\sigma
+\int_0^1\!\int_0^t\kappa\,\mathrm ds\,\mathrm dt\nonumber\\
&=\varepsilon_\mu+\varepsilon_\sigma+\frac{\kappa}{2}
=\Delta_1.
\label{eq:latent_endpoint_error}
\end{align}
\textbf{Step 2: latent error to conditional policy discrepancy.}
For every fixed $\boldsymbol s$, couple
$\boldsymbol\pi^{(1)}_{\boldsymbol\theta}(\cdot\mid\boldsymbol s)$ and
$\boldsymbol\pi^\star(\cdot\mid\boldsymbol s)$ with the same
$(z,\boldsymbol\xi)$, where the optimal endpoint simply ignores
$\boldsymbol\xi$.  Since $\tanh$ is $1$-Lipschitz, this is an admissible
coupling in the definition of $W_2$, and hence
\begin{align}
W_2^2\!\left(
\boldsymbol\pi^{(1)}_{\boldsymbol\theta}(\cdot\mid\boldsymbol s),
\boldsymbol\pi^\star(\cdot\mid\boldsymbol s)\right)
&\leq\mathbb E_{z,\boldsymbol\xi\mid\boldsymbol s}
\!\left[\|\tanh(\widehat a_1)-\tanh(a_1^\star)\|_2^2\right]\nonumber\\
&\leq\mathbb E_{z,\boldsymbol\xi\mid\boldsymbol s}
\!\left[\|\widehat a_1-a_1^\star\|_2^2\right].
\label{eq:conditional_wasserstein_coupling}
\end{align}
Integrating~\eqref{eq:conditional_wasserstein_coupling} with respect to
$\rho$ and invoking~\eqref{eq:latent_endpoint_error} proves~\eqref{eq:one_step_wasserstein_bound}.

\textbf{Step 3: conditional policy discrepancy to return loss.}
The performance-difference identity, written using the normalized discounted
occupancy $d^{(1)}_{\boldsymbol\theta}$, is
\begin{align}
J(\boldsymbol\pi^\star)-J(\boldsymbol\pi^{(1)}_{\boldsymbol\theta})
=\frac{1}{1-\gamma}\,
\mathbb E_{\boldsymbol s\sim d^{(1)}_{\boldsymbol\theta}}
\!\left[
\mathbb E_{\boldsymbol a\sim\boldsymbol\pi^\star}
Q^\star(s,\boldsymbol a)
-\mathbb E_{\boldsymbol a\sim\boldsymbol\pi^{(1)}_{\boldsymbol\theta}}
Q^\star(s,\boldsymbol a)\right].
\label{eq:policy_difference_identity}
\end{align}
Kantorovich--Rubinstein duality and the $K_Q$-Lipschitz property imply that
the inner difference in~\eqref{eq:policy_difference_identity} is at most
$K_QW_1\leq K_QW_2$.  The occupancy-ratio assumption and Cauchy--Schwarz give
\begin{align}
\mathbb E_{d^{(1)}_{\boldsymbol\theta}}[W_2]
&=\mathbb E_{\rho}\!\left[
\frac{\mathrm d d^{(1)}_{\boldsymbol\theta}}{\mathrm d\rho}W_2\right]
\leq\sqrt{C_\rho}\,
\left(\mathbb E_{\rho}[W_2^2]\right)^{1/2}
\leq\sqrt{C_\rho}\Delta_1.
\label{eq:occupancy_transfer_bound}
\end{align}
Substitution into~\eqref{eq:policy_difference_identity} proves~\eqref{eq:one_step_performance_bound}.  If all three error terms vanish,
the coupling agrees almost surely at the endpoint, establishing exact
one-step recovery.
\end{proof}

\section{Details of the Flow Policy}\label{appendix:flow_policy}

In this section, we provide both the theoretical formulation and the practical
implementation of our flow policy. We first show that flow-based policies
admit path score surrogate evaluation through the continuous-time
change-of-variables formulation. We next derive the corresponding
expression and its computation along the flow trajectory. Finally, we detail
the flow policy used in our experiments, including the observation-conditioned
Transformer architecture for parameterizing the velocity field and the
path-wise likelihood surrogate used in practice to provide the log-policy term
for policy optimization.

\subsection{Derivation of path score surrogate}
We employ a factorized policy
$\bm{\pi}_{\theta}(a|s)=\prod_{i=1}^N\pi_{\theta_i}(a^i|s)$,
where each agent $i$ generates its action independently conditioned on the global state $s$. Specifically, we parameterize the policy of agent $i$ with a Transformer-based velocity field $v_{\theta_i}(a_\tau^i,s,\tau)$ and generate actions by evolving an initial Gaussian noise $a_1^i\sim\mathcal{N}(0,I)$ along the reverse-time trajectory from $\tau=1$ to $\tau=0$, and finally $a^i=\mathrm{tanh}(a_0^i)$. To introduce stochasticity into the generation process, we formulate the flow policy as a reverse-time SDE that preserves the desired marginal distribution~\citep{zhang2026sac,liu2025flowgrpo}. For each agent $i$, the reverse-time dynamics are given by
\begin{equation}
\mathrm{d}a_\tau^i =
\left[
v_{\theta_i}(a_\tau^i,s,\tau)
-\frac{\sigma_\tau^2}{2}
\nabla_{a_\tau^i}\log p_\tau(a_\tau^i|s)
\right]\mathrm{d}\tau
+\sigma_\tau\mathrm{d}w,
\end{equation}
where $w$ denotes a Wiener process and $\sigma_\tau$ controls the level of stochasticity during generation. For the rectified flow formulation, the score term can be expressed using the velocity field, yielding
\begin{equation}
\mathrm{d}a_\tau^i =
\left[
v_{\theta_i}(a_\tau^i,s,\tau)
+
\frac{\sigma_\tau^2}{2\tau}
\left(
a_\tau^i+(1-\tau)v_{\theta_i}(a_\tau^i,s,\tau)
\right)
\right]\mathrm{d}\tau
+\sigma_\tau\mathrm{d}w.
\end{equation}
Applying the Euler-Maruyama discretization gives the stochastic action update
\begin{equation}
\label{eq:sde_flow_policy}
a_{\tau-\Delta\tau}^i=
a_\tau^i
-
\left[
v_{\theta_i}(a_\tau^i,s,\tau)
+
\frac{\sigma_\tau^2}{2\tau}
\left(
a_\tau^i+(1-\tau)v_{\theta_i}(a_\tau^i,s,\tau)
\right)
\right]\Delta\tau
+
\sigma_\tau\sqrt{\Delta\tau}\epsilon,
\end{equation}
where $\epsilon\sim\mathcal{N}(0,I)$ provides stochasticity during action generation, and $\Delta\tau>0$ denotes the step size, such that the reverse-time process evolves from $\tau=1$ to $\tau=0$ with $\tau\leftarrow\tau-\Delta\tau$ at each step. Conditioned on $a_\tau^i$ and $s$, the one-step denoised action $a_{\tau-\Delta\tau}^i$ follows a Gaussian distribution:
\begin{equation}
p_{\theta_i}\left(a_{\tau-\Delta\tau}^i\mid a_\tau^i,s\right)
\sim
\mathcal{N}\left(
a_{\tau-\Delta\tau}^i;
\mu_{\theta_i}(a_\tau^i,s,\tau),
\sigma_\tau^2\Delta\tau I
\right),
\end{equation}
where $\mu_{\theta_i}(a_\tau^i,s,\tau)$ denotes the deterministic mean specified by the drift term in~\ref{eq:sde_flow_policy}. Therefore, the one-step transition likelihood admits a closed-form expression, making it tractable to evaluate and optimize.

\subsection{Detailed derivation of the path score surrogate}
Let $\tilde{a}^i=(a_{\tau_0}^i,\ldots,a_{\tau_H}^i)$ denote the reverse-time trajectory of agent $i$, where $\tau_0=1>\tau_1>\cdots>\tau_H=0$ and $a_{\tau_0}^i\sim\mathcal{N}(0,I)$. The final action is obtained by applying an element-wise $\tanh$ transformation,
$a^i=\tanh(a_{\tau_H}^i)=\tanh(a_0^i)$.
Based on the factorized policy
$\bm{\pi}_{\theta}(a|s)=\prod_{i=1}^N\pi_{\theta_i}(a^i|s)$, the joint path density of all agents $\tilde{a}=(\tilde{a}^1,\ldots,\tilde{a}^N)\sim p_{\theta}(\tilde{a}|s)=\prod_{i=1}^N\pi_{\theta_i}(\tilde{a}^i|s)$ is written as:
\begin{equation}
\label{eq:joint_path_density}
p_{\theta}(\tilde{a}|s)
=
\prod_{i=1}^N
\left[
\zeta(a_{\tau_0}^i)
\prod_{h=0}^{H-1}
p_{\theta_i}
\left(
a_{\tau_{h+1}}^i
\mid
a_{\tau_h}^i,s;\Delta\tau_h
\right)
\Vert \mathrm{det}(\mathcal{J}(\mathrm{tanh}(a_{\tau_H}^i))) \Vert^{-1}
\right],
\end{equation}
where $\zeta(\cdot)$ denotes the standard Gaussian base density, and $p_{\theta_i}(\cdot\mid a_{\tau_h}^i,s;\Delta\tau_h)$ is the Gaussian one-step transition induced by the Euler-Maruyama discretization. For a $D$-dimensional vector $a$, since the $\tanh$ transformation is applied independently to each action dimension, its Jacobian determinant admits the closed-form expression
$
\left\Vert\det J(\tanh(a))\right\Vert
=
\prod_{d=1}^{D}
\left(1-(a_d)^2\right)^{-1}.
$
The marginal policy density of $a=(a^1,\ldots,a^N)$ is written as integrating the intermediate pre-activations:
\begin{equation}
\ell_{\theta}(a|s)
=
\int
\cdots
\int
\ell_{\theta}(\tilde{a}|s)
\mathrm{d}a_{\tau_0}^1
\cdots
\mathrm{d}a_{\tau_{H-1}}^1
\cdots
\mathrm{d}a_{\tau_0}^N
\cdots
\mathrm{d}a_{\tau_{H-1}}^N,
\qquad
a^i=\tanh(a_0^i).
\end{equation}
We denote ${\ell}^i\approx\pi_{\theta_i}(a^i|s)$ as the path score surrogate of the agent $i$ and it can be written as   
\begin{equation}
{\ell}^i=
\int
\cdots
\int
p_{\theta_i}(\tilde{a}^i|s)
\mathrm{d}a_{\tau_0}^i
\cdots
\mathrm{d}a_{\tau_{H-1}}^i.
\end{equation}
The path score surrogate enables efficient value estimation and policy optimization compared with the diffusion-based policies, e.g., OMAD~\citep{li2026omad}.

We give a theorem to guarantee that the use of the path score surrogate is reasonable for efficient exploration for value estimation and synchronized policy optimization.

\begin{theorem}[Bounded discrepancy of the path score surrogate] Given that $z^i=(a_{\tau_0}^i,\ldots,a_{\tau_{H-1}}^i)$, assume that the conditional differential entropy of the latent flow path given the final action and state is finite and bounded by
\begin{equation}
\left|
\mathcal{H}_{\theta_i}(Z^i\mid A^i,s)
\right|
\le \varepsilon_i.
\end{equation}
Then the expected discrepancy between the exact action log-likelihood and the path score surrogate satisfies
\begin{equation}
\left|
\mathbb{E}_{p_{\theta_i}(a^i,z^i\mid s)}
\left[
\log \pi_{\theta_i}(a^i\mid s)
-
\ell_{\theta_i}(a^i,z^i\mid s)
\right]
\right|
\le \varepsilon_i.  
\end{equation}
In particular, if the latent path is conditionally concentrated given the final action, i.e.,
$\mathcal{H}_{\theta_i}(Z^i\mid A^i,s)\le\varepsilon_i$
with small $\varepsilon_i$, then the path score provides a close approximation to the exact action log-likelihood in expectation.

For the factorized multi-agent policy
$\bm{\pi}_{\theta}(a\mid s)=\prod_{i=1}^N\pi_{\theta_i}(a^i\mid s)$,
the discrepancy satisfies
\begin{equation}
\left|
\mathbb{E}
\left[
\log \bm{\pi}_{\theta}(a\mid s)
-
\sum_{i=1}^N
\ell_{\theta_i}(a^i,z^i\mid s)
\right]
\right|
\le
\sum_{i=1}^N \varepsilon_i.
\end{equation}

\end{theorem}

\begin{proof}
For any $(a^i,z^i)$ with positive density, Bayes' rule gives
\begin{equation}
p_{\theta_i}(z^i\mid a^i,s)
=
\frac{
p_{\theta_i}(a^i,z^i\mid s)
}{
\pi_{\theta_i}(a^i\mid s)
}.
\end{equation}
Taking logarithms,
\begin{equation}
\ell_{\theta_i}(a^i,z^i\mid s)
=
\log \pi_{\theta_i}(a^i\mid s)
+
\log p_{\theta_i}(z^i\mid a^i,s).
\end{equation}

Therefore,
\begin{equation}
\log \pi_{\theta_i}(a^i\mid s)
-
\ell_{\theta_i}(a^i,z^i\mid s)
=
-\log p_{\theta_i}(z^i\mid a^i,s).
\end{equation}

Taking expectation with respect to
$p_{\theta_i}(a^i,z^i\mid s)$ yields
\begin{equation}
\mathbb{E}
\left[
\log \pi_{\theta_i}(a^i\mid s)
-
\ell_{\theta_i}(a^i,z^i\mid s)
\right]
=
-\mathbb{E}
\left[
\log p_{\theta_i}(z^i\mid a^i,s)
\right]
=
\mathcal{H}_{\theta_i}(Z^i\mid A^i,s).   
\end{equation}

Hence,
\begin{equation}
\left|
\mathbb{E}
\left[
\log \pi_{\theta_i}(a^i\mid s)
-
\ell_{\theta_i}(a^i,z^i\mid s)
\right]
\right|
=
\left|
\mathcal{H}_{\theta_i}(Z^i\mid A^i,s)
\right|
\le \varepsilon_i.  
\end{equation}

For the factorized policy,
\begin{equation}
\log \bm{\pi}_{\theta}(a\mid s)
=
\sum_{i=1}^N
\log \pi_{\theta_i}(a^i\mid s),
\end{equation}

and thus
\begin{equation}
\mathbb{E}
\left[
\log \bm{\pi}_{\theta}(a\mid s)
-
\sum_{i=1}^N
\ell_{\theta_i}(a^i,z^i\mid s)
\right]
=
\sum_{i=1}^N
\mathcal{H}_{\theta_i}(Z^i\mid A^i,s),
\end{equation}

which gives
\begin{equation}
\left|
\mathbb{E}
\left[
\log \bm{\pi}_{\theta}(a\mid s)
-
\sum_{i=1}^N
\ell_{\theta_i}(a^i,z^i\mid s)
\right]
\right|
\le
\sum_{i=1}^N\varepsilon_i.   
\end{equation}

This completes the proof.
\end{proof}

\begin{corollary}[Validity for entropy-regularized policy optimization]
Consider an entropy-regularized objective with temperature $\alpha_i$. Replacing the intractable term
$\log\pi_{\theta_i}(a^i\mid s)$ by the path score surrogate
$\ell_{\theta_i}(a^i,z^i\mid s)$ incurs at most an expected objective discrepancy of
$\alpha_i\varepsilon_i$ for agent $i$. Hence, when the conditional path entropy is small, path-score optimization provides a controlled approximation to maximum-entropy policy optimization.
\end{corollary}

\subsection{Design of the Flow Policy}

\paragraph{Observation-conditioned flow policy.}
For each agent $i$, we encode its local observation into a context
representation and use a Transformer decoder to parameterize the flow velocity.
Specifically, the observation is first normalized and mapped into the model
dimension through a two-layer MLP:
\begin{equation}
c^i
=
E^i_o\left(
\operatorname{BR}(s)
\right),
\end{equation}
where $\operatorname{BR}(\cdot)$ denotes Batch Renormalization and
$E^i_o$ consists of two linear layers with a SiLU activation. The resulting
context vector is reshaped into a singleton sequence and serves as the memory
input to the Transformer decoder.

At flow step $h$, the current latent action $a_h^i\coloneq a_{\tau_h}^i$ is first normalized and
projected into the model dimension. In parallel, the normalized flow time
$\tau_h=1-h/H$ is encoded by a three-layer MLP with SiLU activations. The action and
time embeddings are then added to form the decoder input:
\begin{equation}
z_h^i
=
E_a^i\left(
\operatorname{BR}(a_h^i)
\right)
+
E_\tau^i(\tau_h).
\label{eq:flow_input_embedding}
\end{equation}
Both $z_h^i$ and the observation context $c^i$ contain a single token.

\paragraph{Transformer decoder.}
The input token is processed by $L$ Transformer decoder layers. Each layer
contains a self-attention module, an observation-conditioned cross-attention
module, and a feed-forward network, with residual connections and Batch
Renormalization after each sublayer. For the $l$-th layer, let
$z_h^{i,(l-1)}$ denote the input representation. The computation is
\begin{equation}
\begin{aligned}
u_h^{i,(l)}
&=
z_h^{i,(l-1)}
+
\operatorname{SA}_l\left(
z_h^{i,(l-1)}
\right),\\
\bar{u}_h^{i,(l)}
&=
\operatorname{BR}\left(
u_h^{i,(l)}
\right),\\
q_h^{i,(l)}
&=
\bar{u}_h^{i,(l)}
+
\operatorname{CA}_l\left(
\bar{u}_h^{i,(l)},c^i,c^i
\right),\\
\bar{q}_h^{i,(l)}
&=
\operatorname{BR}
\left(
q_h^{i,(l)}
\right),\\
z_h^{i,(l)}
&=
\bar{q}_h^{i,(l)}
+
\operatorname{FFN}_l
\left(
\bar{q}_h^{i,(l)}
\right),\\
z_h^{i,(l)}
&\leftarrow
\operatorname{BR}
\left(
z_h^{i,(l)}
\right).
\end{aligned}
\label{eq:flow_transformer_layer}
\end{equation}
Here, $\operatorname{SA}_l$ denotes self-attention over the action token,
while $\operatorname{CA}_l$ uses the decoder representation as the query and
the observation context as both key and value. Since both sequences contain a
single token, the self-attention does not perform token-to-token interaction;
instead, the observation-dependent cross-attention injects the local
observation information into the action representation. The feed-forward
network consists of two linear projections with a GELU activation and expands
the hidden dimension by a factor of four.

The multi-head attention modules in each Transformer layer are implemented using scaled dot-product attention. Specifically, for an attention module with $N_h$ heads, the input sequences are first projected into query, key, and value representations. For head $j$, we compute
\begin{equation}
Q_j = X_q W_{Q,j},\qquad
K_j = X_h W_{K,j},\qquad
V_j = X_v W_{V,j},
\end{equation}
and apply scaled dot-product attention:
\begin{equation}
\operatorname{Attn}_j(Q_j,K_j,V_j)
=
\operatorname{softmax}
\left(
\frac{Q_jK_j^\top}{\sqrt{d_h}}
\right)V_j,
\end{equation}
where $d_h$ is the dimension of each attention head. The outputs from all heads are concatenated and mapped back to the model dimension by an output projection.

In our flow policy, this attention operation is used in two different ways. First, the self-attention module applies attention to the action representation itself:
\begin{equation}
\operatorname{SA}_l(z)
=
\operatorname{MHA}_l(z,z,z),
\end{equation}
where the same representation is used as query, key, and value. Second, the observation-conditioned cross-attention module uses the current action representation as the query and the observation context as both key and value:
\begin{equation}
\operatorname{CA}_l(z,c,c)
=
\operatorname{MHA}_l(z,c,c).
\end{equation}
Thus, the cross-attention explicitly conditions the action representation on the local observation.

\paragraph{Velocity and noise prediction.}
After the final Transformer layer, the resulting representation is passed
through two independent linear heads. The first predicts the flow velocity,
\begin{equation}
v_{\theta_i}(a_h^i,o^i,t_h)
=
W_v^i z_h^{i,(L)}+b_v^i,
\end{equation}
while the second predicts the diagonal log standard deviation,
\begin{equation}
r_{\theta_i}(a_h^i,o^i,t_h)
=
W_\sigma^i z_h^{i,(L)}+b_\sigma^i.
\end{equation}
The predicted log standard deviation is bounded through a hyperbolic tangent
and linearly mapped to the prescribed interval $[-5,2]$, after which the noise
scale is obtained as
\begin{equation}
\sigma_{\theta_i}
=
\exp(r_{\theta_i}).
\end{equation}
The resulting velocity is sampled as
\begin{equation}
\hat v_h^i
=
v_{\theta_i}(a_h^i,o^i,t_h)
+
\sigma_{\theta_i}(a_h^i,o^i,t_h)\epsilon_h,
\qquad
\epsilon_h\sim\mathcal N(0,I),
\end{equation}
and the latent action is updated by
\begin{equation}
a_{k+1}^i
=
a_h^i-\hat v_h^i\Delta t,
\qquad
\Delta t=\frac{1}{H}.
\label{eq:flow_rollout}
\end{equation}
Starting from $a_0^i\sim\mathcal N(0,I)$, this iterative flow-matching
procedure produces the final latent action $a_H^i$. The action is then mapped
to the environment action space using an element-wise $\tanh$ transformation
followed by affine scaling:
\begin{equation}
a^i
=
\left[
\tanh(a_H^i)\odot s^i+b^i
\right]\odot m^i,
\end{equation}
where $s^i$ and $b^i$ denote the action scale and bias, respectively, and
$m^i$ is the action-validity mask.

\paragraph{Implemented path score surrogate.}
While the previous formulation defines the policy path score through the
continuous-time change-of-variables integral, we use a path-wise surrogate
during sampling to obtain the log-policy term required by the actor objective.
Specifically, for a $D$-dimensional sampled trajectory
$\{a_h^i,\epsilon_h^i\}_{h=0}^{H-1}$, the sampler accumulates the log-density
contributions from the initial Gaussian latent and the Gaussian perturbations
injected at each flow step. After accounting for the change of variables
induced by the final $\tanh$ transformation, the resulting scalar is
\begin{equation}
\widetilde{\ell}^{~i}
=
\sum_{d=1}^{D}\log\mathcal{N}(a^i_{0,d};0,1)
+\sum_{h=0}^{H-1}\sum_{d=1}^{D}
\log\mathcal{N}(\epsilon^i_{h,d};0,1)
\nonumber-\sum_{d=1}^{D}
\log\left(1-(u^i_d)^2+10^{-6}\right),
\label{eq:path_score_surrogate}
\end{equation}
where $u^i$ denotes the latent variable immediately before the final
$\tanh$ transformation. The first term accounts for the initial Gaussian
latent, the second accumulates the log-density of the Gaussian perturbations
along the flow trajectory, and the last term is the standard
change-of-variables correction associated with the $\tanh$ transformation.

We give a theorem to establish that the gap between the approximation and the real integration is bounded.

\begin{theorem}[Approximation error of the path score]
Let $\ell_{\theta_i}(a^i,\tilde a^i\mid s)$ denote the exact log-density of a sampled flow trajectory, including the Gaussian base density, all Euler--Maruyama transition densities, and the $\tanh$ change-of-variables term. Let $\widetilde{\ell}^{~i}$ be the practical approximation in~\eqref{eq:path_score_surrogate}. Assume that the noise scales satisfy
\begin{equation}
0<\underline{\sigma}
\le
\sigma_{\tau_h}\sqrt{\Delta\tau_h}
\le
\overline{\sigma},
\qquad h=0,\ldots,H-1,
\end{equation}

and that the pre-activation satisfies
\begin{equation}
1-(u_d^i)^2\ge \kappa>0,
\qquad d=1,\ldots,D.
\end{equation}

Then
\begin{equation}
    \ell_{\theta_i}(a^i,\tilde a^i\mid s)
-
\widetilde{\ell}^{~i}
=
-D\sum_{h=0}^{H-1}
\log\!\left(\sigma_{\tau_h}\sqrt{\Delta\tau_h}\right)
+
\sum_{d=1}^{D}
\log
\left(
1+
\frac{10^{-6}}{1-(u_d^i)^2}
\right),
\end{equation}

and hence
\begin{equation}
    \left|
\ell_{\theta_i}(a^i,\tilde a^i\mid s)
-
\widetilde{\ell}^{~i}
\right|
\le
DH\max\{
|\log\underline{\sigma}|,
|\log\overline{\sigma}|
\}
+
D\log\left(1+\frac{10^{-6}}{\kappa}\right).
\end{equation}

In particular, when the noise schedule and discretization are fixed, the first term is a parameter-independent constant. The remaining approximation error is introduced only by the numerical stabilization term $10^{-6}$ and is negligible when $\kappa$ is not too small.
\label{thm:path_score_approx}
\end{theorem}

\begin{proof}
From the Euler--Maruyama update in~\eqref{eq:sde_flow_policy}, we can write
\begin{equation}
    a_{\tau_{h+1}}^i
=
\mu_{\theta_i}(a_{\tau_h}^i,s,\tau_h)
+
\sigma_{\tau_h}\sqrt{\Delta\tau_h}\,
\epsilon_h^i,
\end{equation}

where $\epsilon_h^i\sim\mathcal N(0,I)$. Therefore, the exact transition density is
\begin{equation}
    p_{\theta_i}
\left(
a_{\tau_{h+1}}^i
\mid
a_{\tau_h}^i,s
\right)
=
\mathcal N
\left(
a_{\tau_{h+1}}^i;
\mu_{\theta_i},
\sigma_{\tau_h}^2\Delta\tau_h I
\right).
\end{equation}

Using the reparameterization above,
\begin{equation}
    \log p_{\theta_i}
\left(
a_{\tau_{h+1}}^i
\mid
a_{\tau_h}^i,s
\right)
=
\sum_{d=1}^{D}
\log\mathcal N(\epsilon_{h,d}^i;0,1)
-
D\log
\left(
\sigma_{\tau_h}\sqrt{\Delta\tau_h}
\right).
\end{equation}

Summing over all $H$ transitions and adding the Gaussian base density gives
\begin{equation}
\begin{aligned}
        \ell_{\theta_i}
&=
\sum_{d=1}^{D}\log\mathcal N(a_{0,d}^i;0,1)
+
\sum_{h=0}^{H-1}\sum_{d=1}^{D}
\log\mathcal N(\epsilon_{h,d}^i;0,1)\\
&-
D\sum_{h=0}^{H-1}
\log
\left(
\sigma_{\tau_h}\sqrt{\Delta\tau_h}
\right)
-
\sum_{d=1}^{D}
\log\left(1-(u_d^i)^2\right).
\end{aligned}
\end{equation}

Comparing this expression with~\eqref{eq:path_score_surrogate}, we obtain
\begin{equation}
    \ell_{\theta_i}-\widetilde{\ell}^{~i}
=
-D\sum_{h=0}^{H-1}
\log
\left(
\sigma_{\tau_h}\sqrt{\Delta\tau_h}
\right)
+
\sum_{d=1}^{D}
\log
\left(
\frac{1-(u_d^i)^2+10^{-6}}
{1-(u_d^i)^2}
\right),
\end{equation}

which is equivalent to the stated equality.

Since
\begin{equation}
    1-(u_d^i)^2\ge\kappa,
\end{equation}

we have
\begin{equation}
    0
\le
\log
\left(
1+\frac{10^{-6}}{1-(u_d^i)^2}
\right)
\le
\log\left(1+\frac{10^{-6}}{\kappa}\right).
\end{equation}

Moreover,
\begin{equation}
    \left|
\log
\left(
\sigma_{\tau_h}\sqrt{\Delta\tau_h}
\right)
\right|
\le
\max\{
|\log\underline{\sigma}|,
|\log\overline{\sigma}|
\}.
\end{equation}

Applying the triangle inequality yields the desired bound.
\end{proof}

The path score $\widetilde{\ell}^{i}$ is computed jointly with the action during sampling and serves as a tractable surrogate for the marginal log-policy $\log\pi_{\theta_i}(a^i\mid s)$ in the actor objective. Since it is accumulated along the sampling trajectory, no additional likelihood evaluation is required after action generation.

In the one-step setting, the flow trajectory consists of a single transition from the initial latent $a_1^i$ to the final latent $a_0^i$. The resulting path score contains the Gaussian base density, the stochastic transition density, and the change-of-variables correction induced by the final $\tanh$ transformation. Thus, even with one-step generation, $\widetilde{\ell}^{~i}$ provides a tractable scalar likelihood surrogate that can be directly incorporated into entropy-regularized actor optimization.

Importantly, Theorem~\ref{thm:path_score_approx} shows that the practical approximation $\widetilde{\ell}^{i}$ differs from the exact path score $\ell_{\theta_i}$ only by a parameter-independent normalization term and a bounded numerical stabilization error. Therefore, using $\widetilde{\ell}^{i}$ preserves the relevant optimization behavior while avoiding the intractable marginalization over intermediate flow states.

\subsection{Relationship to Baseline Algorithms}\label{appendix:discussion}

OMAF differs from existing Gaussian policies and diffusion policies in both policy parameterization and likelihood estimation. Compared with Gaussian policies, such as HARL~\citep{liumaximum,zhong2024heterogeneous}, OMAF parameterizes the policy through a flow velocity field rather than directly specifying a fixed-form Gaussian distribution. Since the marginal likelihood of a flow policy is generally intractable, we use the path score surrogate introduced in Section~\ref{method:flow} to provide a tractable alternative for entropy-regularized policy optimization. More importantly, even in the one-step setting, the resulting action distribution is fundamentally different from a Gaussian policy. Specifically, the velocity network takes a Gaussian latent variable as input and transforms it through a nonlinear neural network, allowing the induced action distribution to exhibit substantially richer and more complex distributional structures than those obtained by directly predicting a Gaussian mean and standard deviation. Thus, the expressiveness of OMAF arises not from explicitly increasing the number of distributional parameters, but from the nonlinear transformation of Gaussian noise through the learned velocity field.

Compared with diffusion policies, e.g., OMAD~\citep{li2026omad}, OMAF differs primarily in the quantity learned to parameterize the generative process. Diffusion policies learn a score function, typically characterized by $\nabla_{a_{\tau}} \log p_\tau(a_{\tau},s,\tau)$, whereas OMAF directly learns the velocity field governing the transport of samples along the flow trajectory. Under appropriate transport formulations, the velocity and score can be closely related, and both approaches generally require a trajectory-based surrogate rather than directly evaluating the marginal likelihood of the generated action. Consequently, their generation processes share a similar structure: both transform an initial noisy sample into a terminal action through a learned continuous-time dynamics. The key distinction is that the learned velocity field in OMAF enables effective few-step, and in our setting one-step, action generation, whereas diffusion policies typically rely on iterative denoising across multiple steps. OMAF therefore retains the flexible generative modeling of noise-to-action transport while substantially reducing the computational cost of action generation.

Existing flow-based reinforcement learning methods, such as SAC-Flow~\citep{zhang2026sac} and FlowRL~\citep{lyu2026flow}, primarily focus on single-agent decision making, where a flow policy models the action distribution conditioned on the observation of an individual agent. In contrast, OMAF is designed specifically for multi-agent reinforcement learning and adopts a centralized training with decentralized execution (CTDE) framework to model coordinated behaviors among multiple agents.

This CTDE formulation is particularly important for flow-based policies in multi-agent settings. Although each individual flow policy generates actions locally, the centralized training process provides a global learning signal that allows the policies to account for the effects of other agents' actions. Consequently, OMAF can learn coordinated behaviors while preserving decentralized execution and avoiding the need for global information at deployment time.

Another key distinction is that OMAF is designed around one-step flow-based action generation. Rather than directly extending a single-agent flow policy to multiple agents, OMAF jointly optimizes the decentralized flow policies with the centralized critic under the multi-agent CTDE framework. This design enables expressive continuous action modeling while substantially reducing the sampling overhead associated with iterative generative policies. These differences make OMAF particularly suited to online multi-agent coordination, where both expressive action distributions and efficient interaction with the environment are essential.

\section{Experimental Details}\label{appendix:exp}
\subsection{Experimental Environments}\label{appendix:env}

\begin{figure*}[ht]
\begin{center}
\centering
\subfloat[Cooperative Navigation]{
\includegraphics[width=0.2\textwidth]{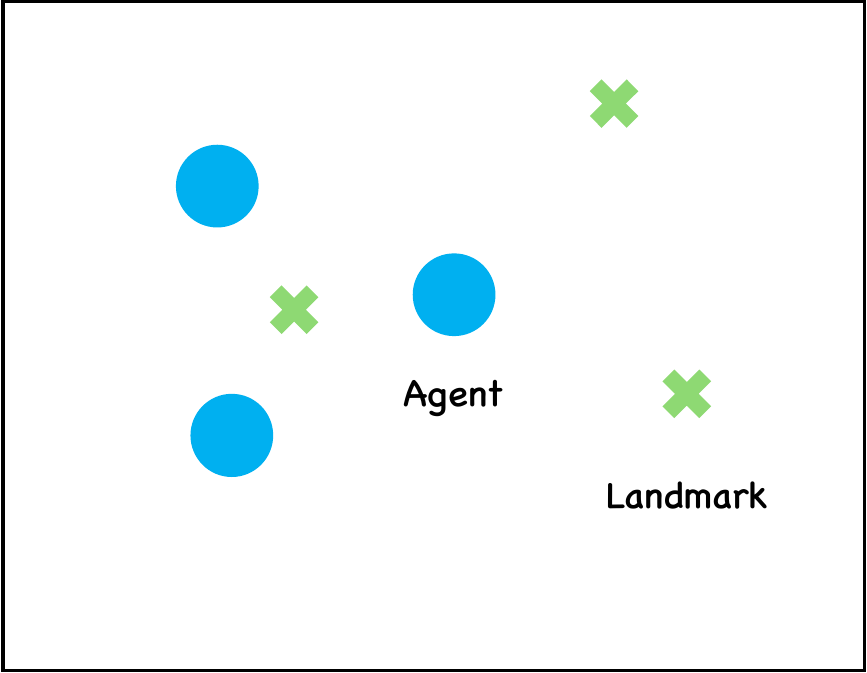}
\label{Fig3-1}
}
\quad
\subfloat[Ant]{
\includegraphics[width=0.2\textwidth]{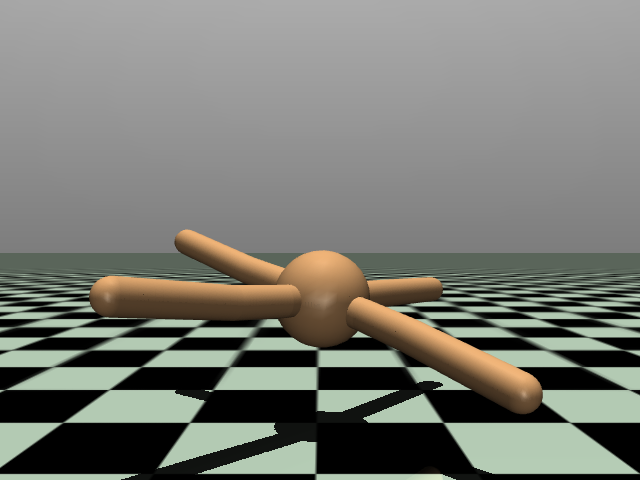}
\label{Fig3-3}
}
\quad
\subfloat[HalfCheetah]{
\includegraphics[width=0.2\textwidth]{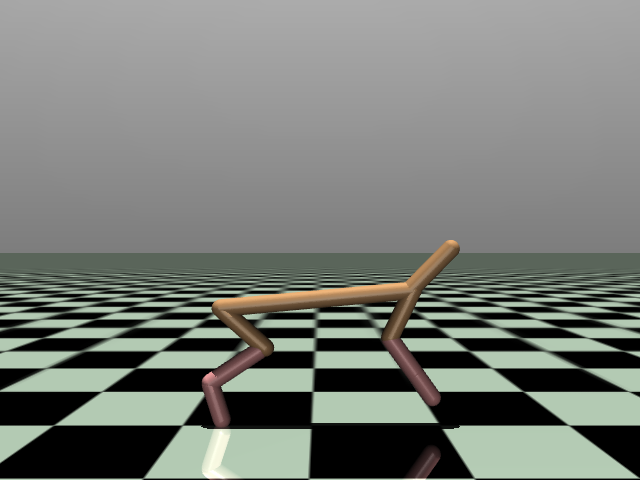}
\label{Fig3-4}
}
\quad
\subfloat[CoupledHalfCheetah]{
\includegraphics[width=0.2\textwidth]{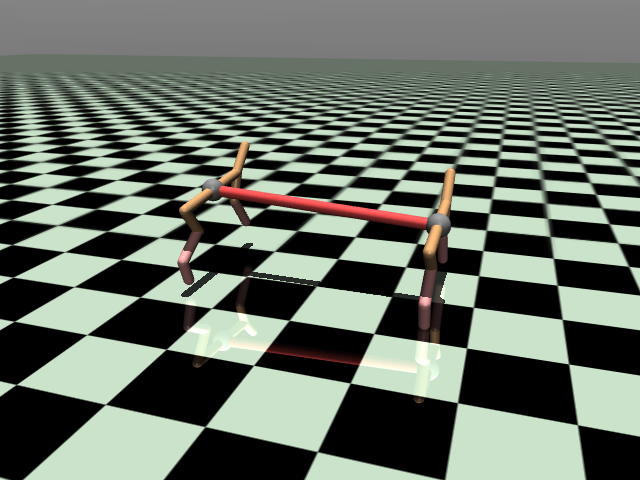}
\label{Fig3-7}
}
\quad
\subfloat[Walker2d]{
\includegraphics[width=0.2\textwidth]{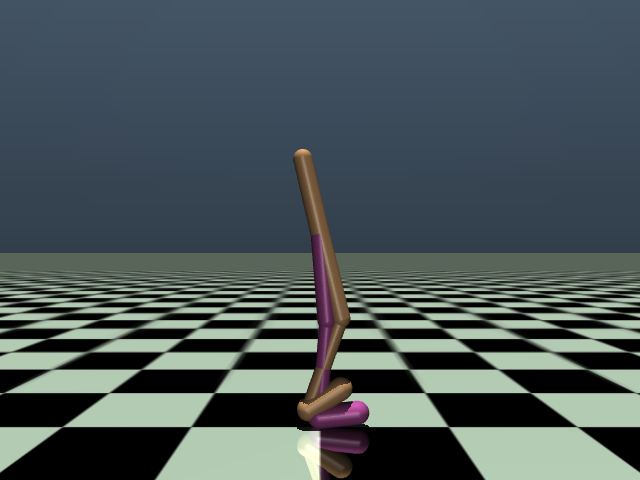}
\label{Fig3-5}
}
\quad
\subfloat[Hopper]{
\includegraphics[width=0.2\textwidth]{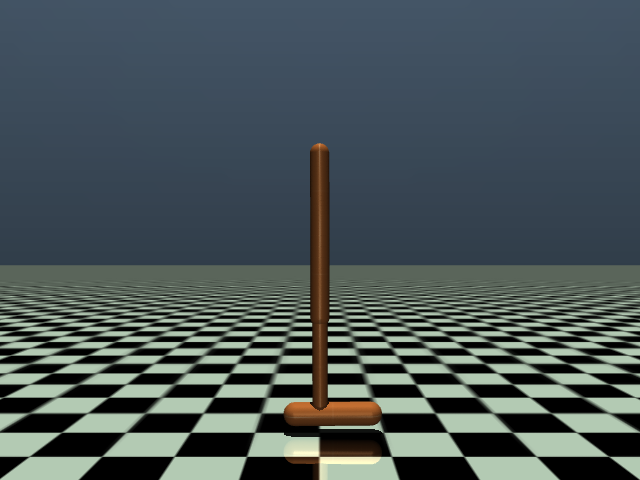}
\label{Fig3-2}
}
\quad
\subfloat[Swimmer]{
\includegraphics[width=0.2\textwidth]{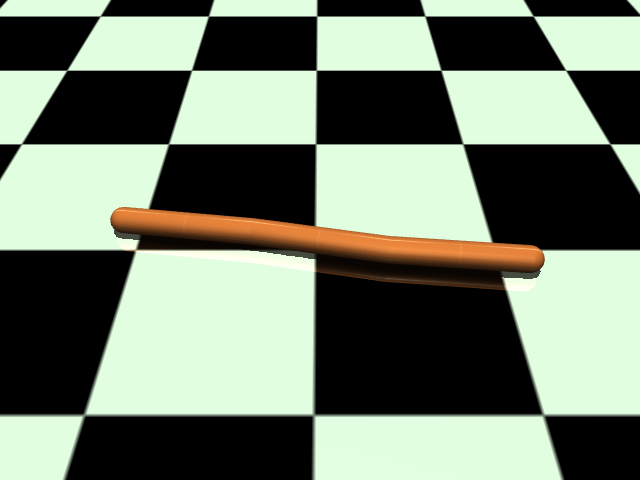}
\label{Fig3-6}
}
\caption{Multi-agent particle environments (MPE) and Multi-agent MuJoCo (MAMuJoCo).}
\label{fig:environments}
\end{center}
\end{figure*}

We conduct experiments on two widely used multi-agent continuous-control
benchmarks, Multi-Agent Particle Environments (MPE)~\citep{lowe2017multi}\footnote{https://pettingzoo.farama.org/environments/mpe/}
and Multi-Agent MuJoCo (MAMuJoCo)~\citep{peng2021facmac}\footnote{https://robotics.farama.org/envs/MaMuJoCo/}. In MPE, we consider
one representative task, \textit{Cooperative Navigation}. Cooperative Navigation requires multiple agents
to jointly cover a set of landmarks while avoiding collisions
(Figure~\ref{Fig3-1}).

For MAMuJoCo, we consider these locomotion environments:
\textit{Ant} (Figure~\ref{Fig3-3}), \textit{HalfCheetah}
(Figure~\ref{Fig3-4}), \textit{CoupledHalfCheetah}
(Figure~\ref{Fig3-7}), \textit{Walker2d} (Figure~\ref{Fig3-5}), and
\textit{Swimmer} (Figure~\ref{Fig3-6}). MAMuJoCo converts conventional
single-agent MuJoCo control tasks into multi-agent settings by assigning
different subsets of the robot's actuators to different agents. The resulting
agents jointly control the same physical system, making coordinated
continuous control essential for successful locomotion. The configurations
used in our experiments are summarized below.

\textbf{Ant} ($2\times4$, $2\times4$d, $4\times2$).
Ant is a four-legged robot with coupled three-dimensional dynamics. The
different agents control disjoint subsets of the robot's actuators, requiring
coordinated limb movements to maintain stability while producing forward
locomotion.

\textbf{HalfCheetah} ($2\times3$, $6\times1$).
HalfCheetah is a planar articulated robot whose actuators are distributed
among multiple agents. Effective locomotion depends on coordinating the
different body segments to generate consistent forward movement.

\textbf{CoupledHalfCheetah}.
CoupledHalfCheetah consists of two HalfCheetah robots connected by an elastic
tendon, requiring agents to coordinate their independently controlled robots
to achieve efficient forward locomotion.

\textbf{Hopper} ($3\times 1$).
Hopper is a two-dimensional one-legged robot controlled through three joint
torques, with the goal of maintaining balance while hopping forward. In the
multi-agent setting, the three joints are partitioned among three agents,
with each agent controlling one joint.

\textbf{Walker2d} ($2\times 3$).
Walker2d is a two-legged robot with an upright-body constraint. Since control
is distributed across agents, successful locomotion requires jointly
maintaining body stability and generating forward motion.

\textbf{Swimmer} ($2\times 1$).
Swimmer is an articulated robot operating in a fluid-like environment. Its
forward movement relies on coordinated motion across multiple connected
body segments, making temporal coordination between agents important for
effective locomotion.

Compared with the conventional single-agent MuJoCo formulation, MAMuJoCo
introduces decentralized control by distributing the robot's actuators across
multiple agents. Although the agents interact through a shared physical
system, each agent controls only a subset of the action dimensions. As a
result, the effect of an individual action depends on the states and actions
of the other agents, creating strong physical coupling and coordination
requirements. These characteristics make MAMuJoCo a useful benchmark for
evaluating multi-agent policy learning in high-dimensional continuous
control.

\begin{table}[ht]
\centering
\caption{Hyperparameters used across all tasks in the MPE and MAMuJoCo environments.}
\label{tab:hyperparameters}
\renewcommand{\arraystretch}{1.25}
\setlength{\tabcolsep}{0pt}

\begin{tabular}{
    @{}
    >{\centering\arraybackslash}p{0.62\linewidth}
    >{\centering\arraybackslash}p{0.15\linewidth}
    @{}
}
\toprule
\textbf{Name of the Hyperparameter} & \textbf{Value} \\
\midrule

Timesteps to Start Learning $L_{\mathrm{init}}$
& 50000 \\
\midrule

Critic Soft Update Coefficient $\rho$
& 0.0 \\
\midrule

Discount Factor $\gamma$
& 0.99 \\
\midrule

Policy Delay $d_l$
& 3 \\
\midrule

Batch Size
& 512 \\
\midrule

Replay Buffer Size $\lvert\mathcal{D}\rvert$
& 1000000 \\
\midrule

Initial Entropy Coefficient $\alpha$
& 0.2 \\
\midrule

Target Entropy $\mathcal{H}_{\mathrm{target}}$
& 0 \\
\midrule

Number of Critics
& 2 \\
\midrule

Number of Samples for Softmax Q-Value Estimation $K$
& 10 \\
\midrule

Number of Flow Steps $H$
& 1 \\
\midrule

Transformer Hidden Dimension
& 96 \\
\midrule

Number of Transformer Layers
& 2 \\
\midrule

Number of Attention Heads
& 4 \\
\midrule

Batch Normalization Momentum
& 0.99 \\
\midrule

Optimizer
& Adam \\
\midrule

Adam $\beta_1$
& 0.5 \\
\midrule

Adam $\beta_2$
& 0.99 \\

\bottomrule
\end{tabular}
\end{table}

\subsection{Hyperparameters}\label{appendix:parameters}
For the critic, we employ a centralized Multi-Layer Perceptron (MLP) to
estimate a scalar $Q$-value from the joint state and joint action. The
state-action input is concatenated and normalized using Batch Normalization
before being fed into the network. The critic consists of two hidden layers
with $1024$ units per layer and ReLU activations, followed by a single scalar
output. We maintain two critics to reduce estimation bias and improve the
stability of value learning.

For the actor, we employ the flow-based policy described in
Appendix~\ref{appendix:flow_policy}. The flow velocity is parameterized by an
observation-conditioned Transformer, which takes the current latent action,
flow time, and local observation as inputs. We use a one-step flow
transformation ($H=1$), while the Transformer architecture and the
corresponding policy optimization details are provided in
Appendix~\ref{appendix:flow_policy}.

To stabilize critic training, we maintain a target policy network for
generating next-step actions. The target policy is updated according to an
exponential moving average with coefficient $\rho$. In our experiments,
$\rho=0$, corresponding to a hard-update scheme in which the target policy
is updated directly rather than through a soft moving average. The critic
uses a discount factor of $\gamma=0.99$ for computing discounted target
values.

Table~\ref{tab:hyperparameters} summarizes the hyperparameters shared across
all tasks, including training settings, network architectures, flow
configuration, normalization, and optimizer parameters. Table
\ref{tab:task_hyperparameters} reports task-specific optimization
hyperparameters, including the policy and critic learning rates, gradient
clipping thresholds, and the temperature $\beta$ used for softmax
$Q$-value estimation. Unless otherwise specified, all experiments are
conducted with $3$ independent random seeds.

\begin{table}[ht]
\centering
\caption{Task-specific optimization hyperparameters for the MPE and MAMuJoCo environments.}
\label{tab:task_hyperparameters}
\renewcommand{\arraystretch}{1.25}
\setlength{\tabcolsep}{5pt}
\resizebox{\linewidth}{!}{
\begin{tabular}{
>{\centering\arraybackslash}p{0.34\linewidth}
>{\centering\arraybackslash}p{0.13\linewidth}
>{\centering\arraybackslash}p{0.13\linewidth}
>{\centering\arraybackslash}p{0.14\linewidth}
>{\centering\arraybackslash}p{0.18\linewidth}}
\toprule
Task
& Policy Learning Rate
& Critic Learning Rate
& Gradient Clip Norm
& Softmax Temperature $\beta$ \\
\midrule

Cooperative Navigation 5 Agents
& $3 \times 10^{-6}$
& $1 \times 10^{-3}$
& $10000$
& $1 \times 10^{-1}$ \\
\midrule

Ant $2 \times 4$
& $3 \times 10^{-4}$
& $3 \times 10^{-3}$
& $10000$
& $1 \times 10^{-2}$ \\
\midrule

Ant $2 \times 4$d
& $1 \times 10^{-4}$
& $3 \times 10^{-3}$
& $10000$
& $1 \times 10^{-2}$ \\
\midrule

Ant $4 \times 2$
& $3 \times 10^{-5}$
& $3 \times 10^{-3}$
& $10000$
& $1 \times 10^{-2}$ \\
\midrule

HalfCheetah $2 \times 3$
& $3 \times 10^{-4}$
& $3 \times 10^{-3}$
& $10$
& $1 \times 10^{-1}$ \\
\midrule

HalfCheetah $6 \times 1$
& $3 \times 10^{-4}$
& $1 \times 10^{-3}$
& $1$
& $1 \times 10^{-2}$ \\
\midrule

Walker2d $2 \times 3$
& $3 \times 10^{-5}$
& $3 \times 10^{-3}$
& $10000$
& $1 \times 10^{-3}$ \\
\midrule

Hopper $3 \times 1$
& $3 \times 10^{-4}$
& $1 \times 10^{-3}$
& $10000$
& $1 \times 10^{-2}$ \\
\midrule

CoupledHalfCheetah
& $3 \times 10^{-4}$
& $1 \times 10^{-3}$
& $10000$
& $3 \times 10^{-1}$ \\
\midrule

Swimmer $2 \times 1$
& $1 \times 10^{-3}$
& $3 \times 10^{-4}$
& $0.1$
& $1 \times 10^{-2}$ \\

\bottomrule
\end{tabular}}
\end{table}

\subsection{Additional Experimental Results}
\subsubsection{Additional Ablation Studies}
In this section, we perform ablation studies to investigate the impact of three key design choices in OMAF: the denoising steps of the flow policy, the Transformer depth of the velocity network, and the softmax temperature for Q-value estimation. We systematically vary these components to understand their individual effects on policy performance and training efficiency.

\begin{figure}[ht]
    \centering
    \begin{subfigure}[b]{0.31\textwidth}
        \centering
        \includegraphics[width=\textwidth]{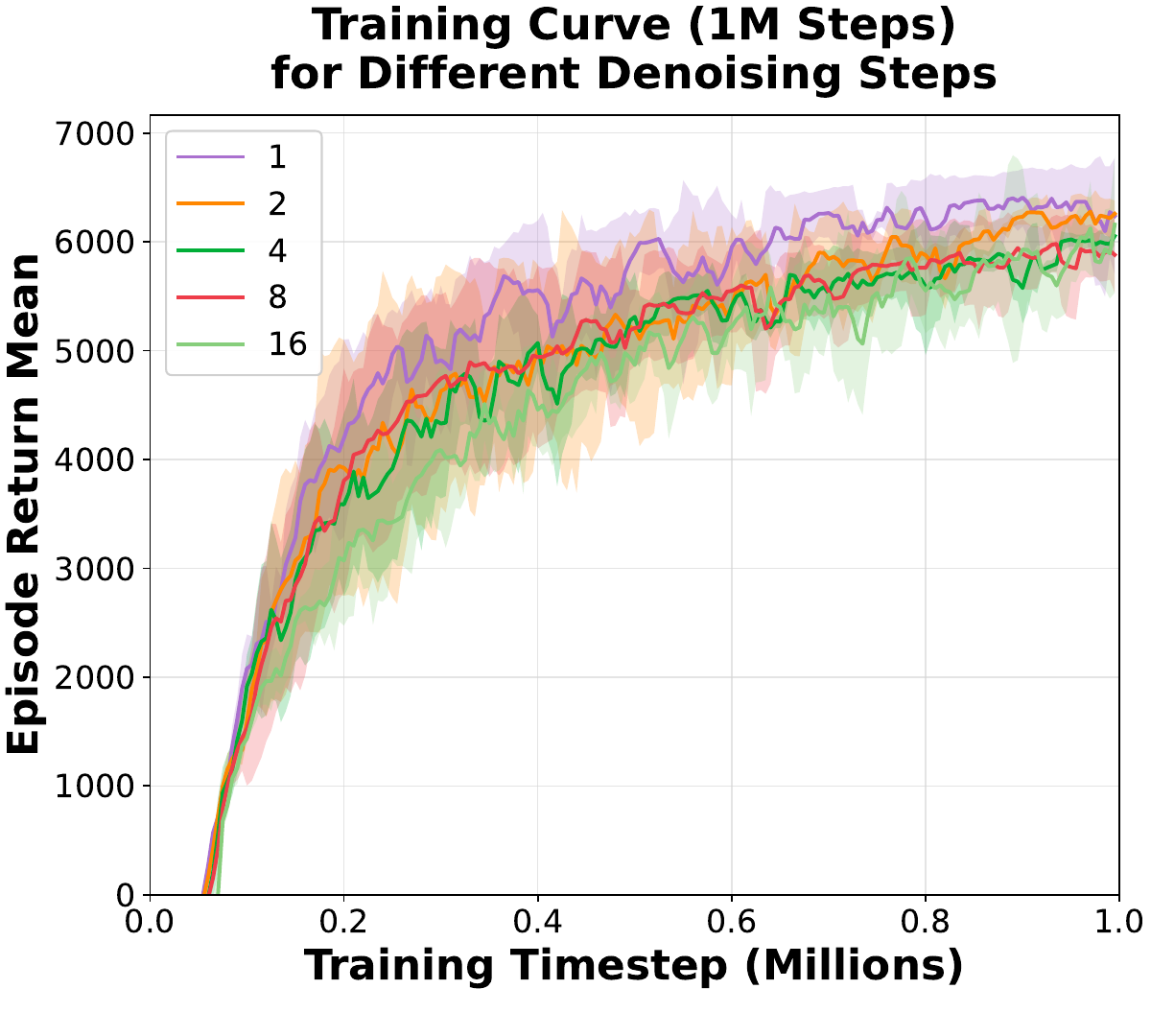}
        \caption{Ablation on denoising steps.}
        \label{fig:denoising_steps}
    \end{subfigure}
    \begin{subfigure}[b]{0.31\textwidth}
        \centering
        \includegraphics[width=\textwidth]{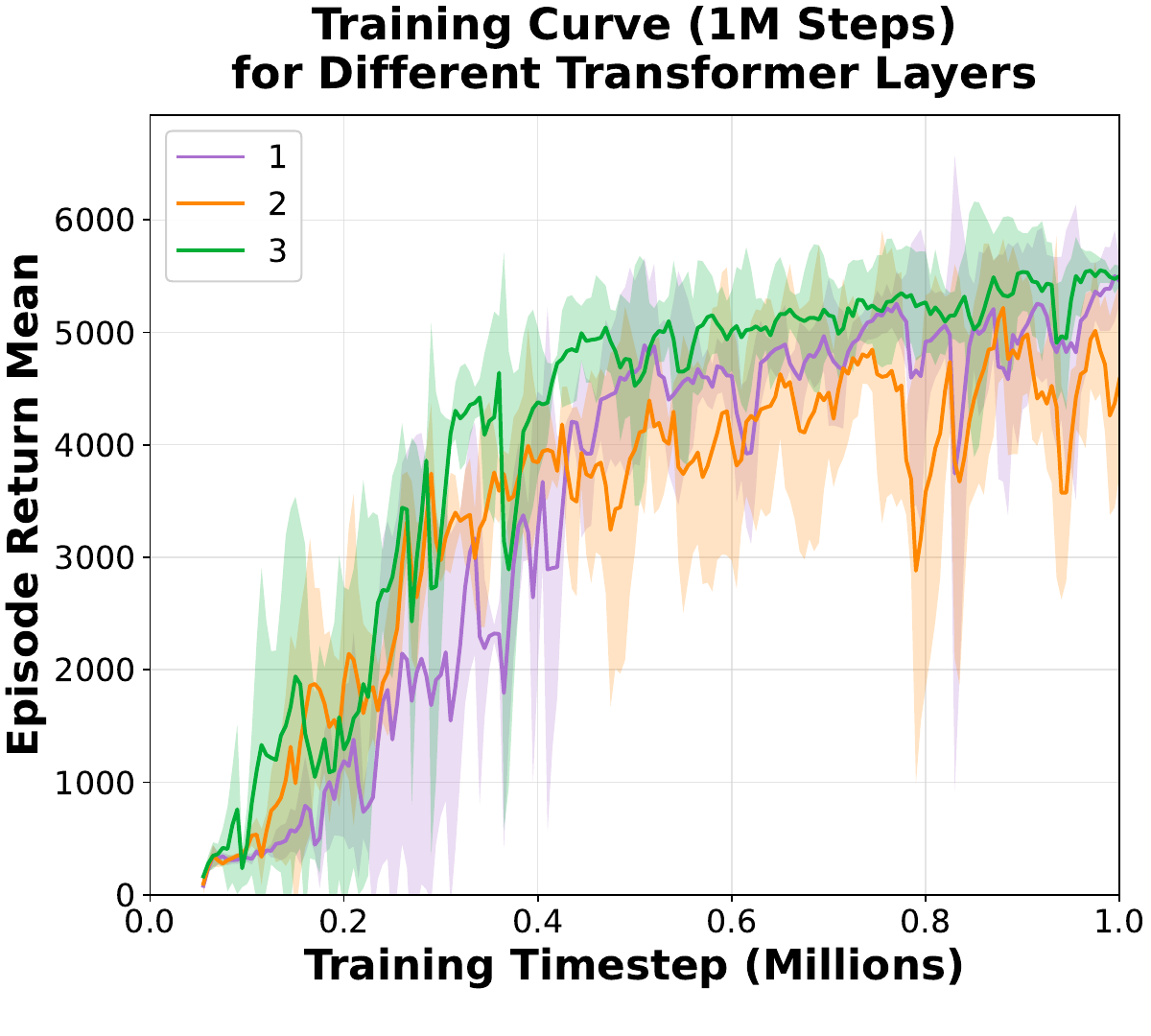}
        \caption{Ablation on denoising steps.}
        \label{fig:n_layers}
    \end{subfigure}
    \begin{subfigure}[b]{0.31\textwidth}
        \centering
        \includegraphics[width=\textwidth]{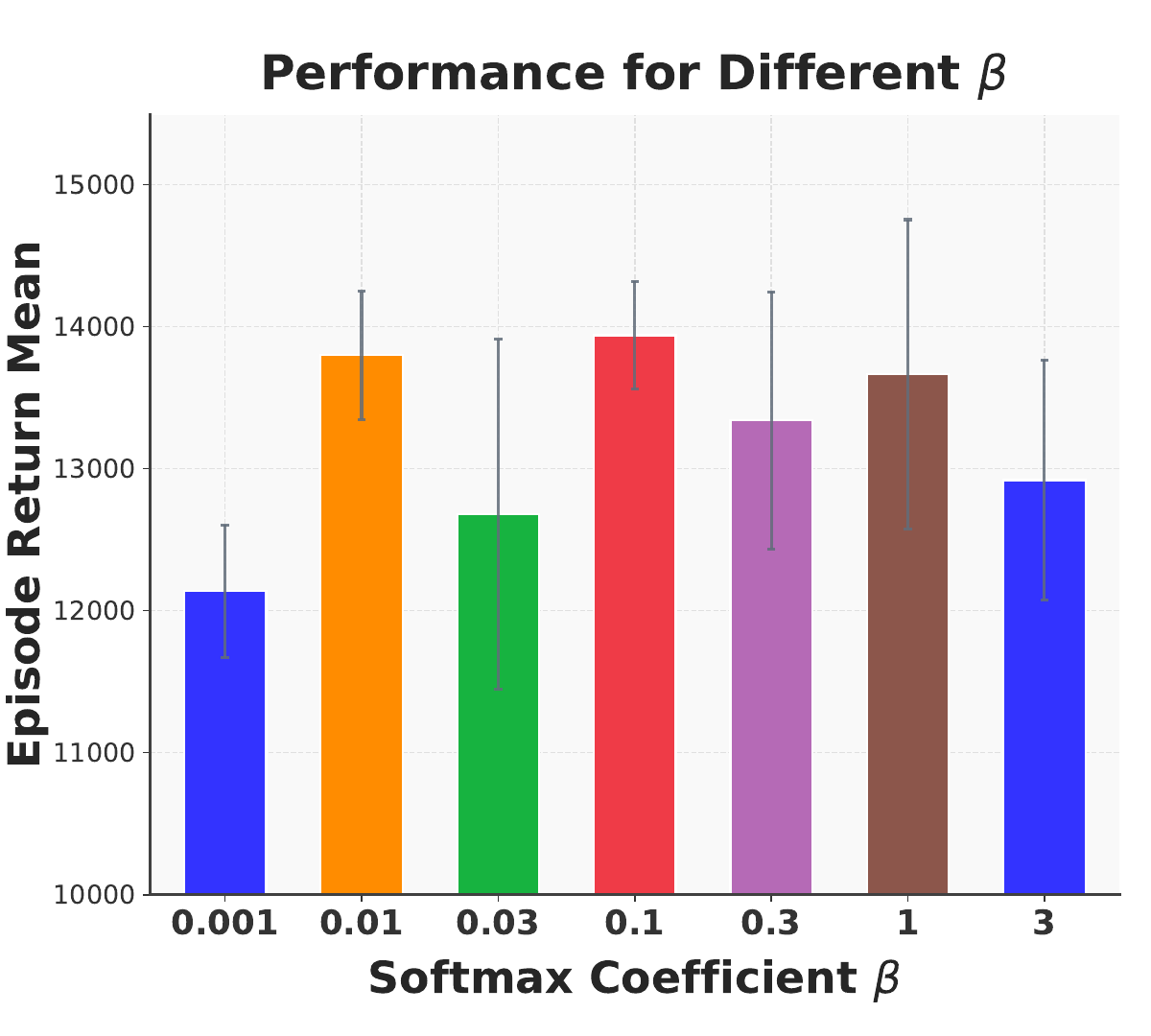}
        \caption{Ablation on softmax $\beta$.}
        \label{fig:softmax_beta}
    \end{subfigure}
    \caption{Analysis of OMAF design choices on Ant $2\times4$. 
    (a) Performance with different numbers of denoising steps over 1 million training steps, demonstrating the effectiveness of one-step denoising for online policy optimization. (b) Performance with different numbers of Transformer depth over 1 million training steps. (c) Performance under different softmax temperature settings.}
    \label{fig:ablation_analysis}
\end{figure}

\paragraph{Effect of Denoising Steps.} Figure~\ref{fig:denoising_steps} compares OMAF with different numbers of denoising steps on the Ant2x4 environment over 1M training steps. All variants are able to learn effective policies and achieve more than 6,500 episodic returns, demonstrating the robustness of the proposed flow policy across different denoising configurations. Notably, the one-step variant achieves the best performance, while requiring only a single function evaluation to generate an action. This result highlights the efficiency of one-step denoising for online policy optimization. In contrast, increasing the number of denoising steps does not lead to consistent performance improvements. A potential explanation is that our policy is updated through Q-value guidance, where the policy gradient is backpropagated through the generated action. With multiple denoising steps, the action depends on a longer sequence of flow transformations, which may introduce additional optimization effects and make the Q-guided policy update less direct. Therefore, one-step flow generation provides a particularly favorable trade-off between expressive policy modeling and efficient gradient-based policy optimization.
\paragraph{Impact of Transformer Depth.} To evaluate how network architecture depth influences policy learning, we compare OMAF across different numbers of Transformer layers ($1$, $2$, and $3$). As shown in the training curves, a shallower network (depth = 1) yields a simpler policy, but its limited representation capacity slows down early learning and bounds the overall return. Increasing the depth enhances policy expressiveness, with depth = 3 achieving faster convergence and superior final performance. However, deeper architectures also introduce higher training difficulty and potential optimization instability, as reflected in the variance observed with intermediate settings. These results emphasize the need to strike a balance between policy expressiveness and optimization stability when configuring the network depth. 
\paragraph{Effect of Softmax Temperature.} We further investigate the effect of the softmax temperature $\beta$ on the HalfCheetah $2\times3$ environment with $\beta\in\{0.001,0.01,0.03,0.1,0.3,1,3\}$. The experimental results are presented in the bar plot in Fig.~\ref{fig:softmax_beta}. The results indicate that the temperature needs to be appropriately selected to achieve effective softmax-based target value correction. In particular, moderate temperature values provide a better balance between smoothing the Q-values and preserving their relative differences, which facilitates more effective Q-guided policy optimization.

\subsubsection{Exploration Details}

To examine the exploration benefits of our expressive flow policies, we analyze the state distributions collected in the replay buffers after the first 250K training steps on Ant $2\times4$~\citep{li2026omad}. We discretize the 2D state space (dimensions 23 and 13) over $[-4,8.5]\times[-17.5,17.5]$ with a grid interval of $0.5$, yielding 1750 bins. OMAF visits 1,162 distinct bins, achieving the broadest state-space coverage, compared with 958 for OMAD, 942 for HATD3, 861 for MAFlowRL, 826 for HASAC, and 819 for MAMFPO. The uniquely explored regions further show that OMAF's expressive entropy-regularized flow policy enables broader exploration and helps avoid getting trapped in local optima.

\begin{figure}[ht]
    \begin{center}
    \includegraphics[width=0.9\textwidth]{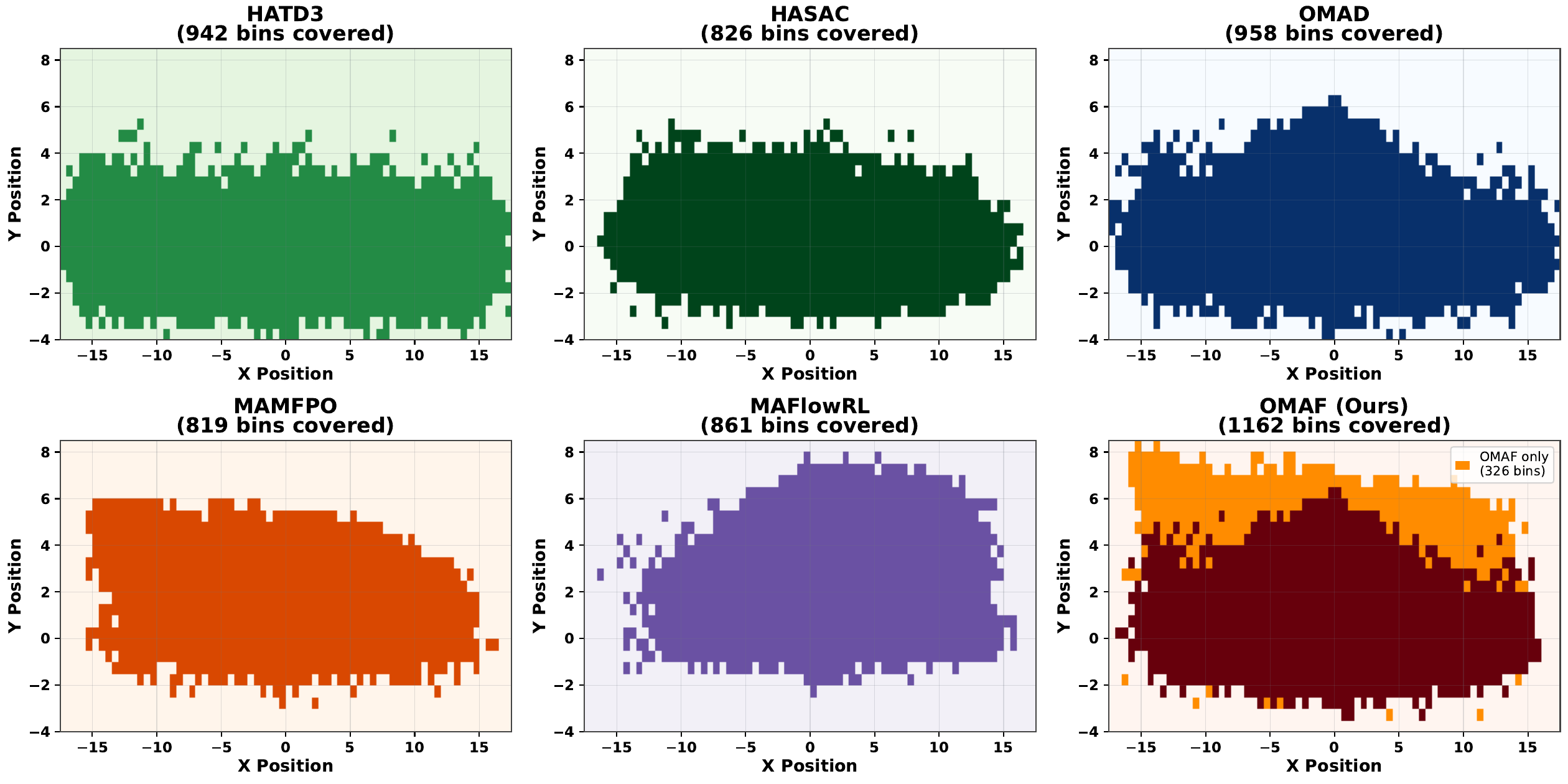}

    \caption{State coverage comparison on representative dimensions (23 and 13) at 250K steps. We visualize the state occupancy within the replay buffers for HATD3, HASAC, diffusion-based OMAD, flow-based MAFlowRL and MAMFPO, and OMAF. Colored regions indicate visited states. OMAF achieves the broadest coverage, where orange regions are uniquely explored by OMAF compared with the diffusion-based OMAD, demonstrating superior exploration.}
    \end{center}
\end{figure}

\subsubsection{Scalability}

As the number of agents grows, multi-agent policy learning becomes increasingly challenging due to the growing dimensionality of joint observations and actions, as well as the stronger coordination dependencies among agents. These challenges are particularly pronounced for expressive generative policies, where action generation may introduce additional computational overhead. In conventional diffusion-based policies, multiple denoising iterations can further increase the cost of generating joint actions, making efficient scaling an important consideration for large-agent systems. Moreover, directly increasing the number of agents in physics-based environments such as MAMuJoCo changes the underlying system dynamics, making controlled scalability studies difficult. We therefore evaluate OMAF on Cooperative Navigation, where the number of agents can be varied while preserving the same underlying task structure.

Table~\ref{tab:cooperative_navigation_scalability} reports the performance with 10 agents and 20 agents. OMAF outperforms the considered baselines, including conventional centralized-critic methods and existing flow- or diffusion-based approaches. The relatively small variance also indicates stable coordination among a larger number of agents. More importantly, OMAF generates actions using a single flow step, avoiding the repeated denoising procedure required by multi-step generative policies. This substantially reduces the computation associated with joint action generation as the number of agents increases. Overall, the results demonstrate that OMAF can maintain effective multi-agent coordination in a substantially larger system while retaining an efficient action-generation mechanism, providing a practical basis for scaling flow-based MARL to larger agent populations.

\begin{table}[ht]
    \centering
    \caption{Performance comparison on Cooperative Navigation with $10$ and $20$ agents.}
    \label{tab:cooperative_navigation_scalability}
    \resizebox{\linewidth}{!}{
    \begin{tabular}{c c c c c c c}
        \toprule
        Task
        & HATD3
        & HASAC
        & MAMFPO
        & MAFlowRL
        & OMAD
        & \textbf{OMAF (Ours)} \\
        \midrule

        \makecell{Cooperative Navigation \\ $N=10$}
        &  $-464.3\pm16.8$
        &  $-471.3\pm14.0$
        & $-478.4 \pm 8.3$
        & $-467.2\pm 23.4$
        & $-445.1\pm3.3$
        & $\mathbf{-429.4 \pm 8.3}$ \\
        \midrule

        \makecell{Cooperative Navigation \\ $N=20$}
        & $-1445.1 \pm 41.3$
        & $-1452.1\pm 18.3$
        & $-1680.7 \pm 45.8$
        & $-1540.6 \pm 56.1$
        & $-1801.4 \pm 36.8$
        & $\mathbf{-1433.4 \pm 14.7}$ \\
        \bottomrule
    \end{tabular}
    }
\end{table}

\subsection{Discussion}

\subsubsection{Training and Evaluation Protocol}

For a fair and consistent comparison, we evaluate all algorithms using three independent random seeds, uniformly fixed to 0, 42, and 3407 across all methods. All other experimental settings are kept consistent with the respective algorithm implementations to ensure that the comparison is not affected by differences in random initialization. At each evaluation point, we roll out the current policy for 10 episodes and use the average episodic return as the evaluation performance; we then report the mean and standard deviation across the three random seeds.

To evaluate performance under a limited interaction budget, we compare OMAF with all baseline methods within 2M training steps. This evaluation protocol emphasizes the early-stage learning performance of each method and allows us to assess how effectively the algorithms utilize limited environment interactions. Under this setting, OMAF achieves strong performance within substantially fewer training steps, demonstrating its advantage in the early stages of online learning.

We further examine the learning curves over longer training horizons. While several baseline methods continue to improve with additional environment interactions, OMAF reaches its converged performance earlier. This indicates that the baselines can benefit from prolonged training, but require more interaction steps to realize their full performance. In contrast, OMAF is built upon an efficient generative policy that enables effective policy learning with fewer environment interactions, highlighting its interaction efficiency for online multi-agent reinforcement learning.

\subsubsection{Discussion about the Performance Metric in the Radar Plot}

\paragraph{Performance} The Performance dimension in the radar plot is computed from the peak evaluation returns across ten benchmark environments: Ant $2\times4$, Ant $4\times2$, Ant $2\times4$d, Walker2d $2\times3$, HalfCheetah $2\times3$, HalfCheetah $6\times1$, Coupled HalfCheetah, Hopper $3\times1$, Swimmer $2\times1$, and Cooperative Navigation. For each environment and algorithm, we first identify the peak evaluation return and average it across the three random seeds (0, 42, and 3407).

To account for the different reward scales across environments, we independently normalize the mean peak returns of all algorithms within each environment using min-max normalization:
\begin{align}
    \frac{R_{e,m}-\min_{m'}R_{e,m'}}
{\max_{m'}R_{e,m'}-\min_{m'}R_{e,m'}}
\times 100,
\end{align}

where $R_{e,m}$ denotes the mean peak return of method $m$ on environment $e$. A higher return always corresponds to a higher normalized score. The Performance score is then obtained by averaging the normalized scores over the ten environments.

Following this procedure, the Performance scores shown in the radar plot are 99.54 for OMAF, 81.60 for OMAD, 42.12 for MAMFPO, 39.13 for MAFlowRL, 11.67 for HASAC, and 6.24 for HATD3. 

\paragraph{Training Time} The Training Time dimension measures the wall-clock training efficiency on Ant $2\times4$ over 2M training steps, reported as seconds per 100K training steps. The measured times are 84.744, 97.602, 1427.891, 1701.290, 1807.912, and 2981.520 s/100K steps for HASAC, HATD3, OMAF, OMAD, MAMFPO, and MAFlowRL.

For visualization, the training times are mapped to a 0-100 scale such that a higher score indicates lower computational cost. This yields Training Time scores of 99.41, 99.32, 90.00, 88.09, 87.34, and 79.12, respectively. Notably, OMAF achieves the highest training efficiency among the evaluated generative-policy-based methods, requiring less wall-clock training time than OMAD, MAMFPO, and MAFlowRL under the same experimental setting.

\paragraph{Sample Efficiency} The Sample Efficiency dimension measures the relative number of environment interactions required to reach a specified target return. For each environment, we define a target return threshold and record the number of training steps at which the mean learning curve across the three random seeds reaches this threshold for the first time. Methods that do not reach the target within the evaluation horizon are assigned a value of 1.0. The target return thresholds are 3510.4 for Ant 2x4, 2500 for Ant 4x2, 3620.8 for Ant 2x4d, 3370.7 for Walker2d 2x3, 5000 for HalfCheetah 2x3, 10000 for HalfCheetah 6x1, 1402.3 for Coupled HalfCheetah, 2800 for Hopper 3x1, 47.0 for Swimmer 2x1, and -120 for Cooperative Navigation 5.

For each environment, we compute the ratio based on the number of training steps required to reach the corresponding threshold and then average the ratios across the ten environments. This gives average ratios of 9.51 for OMAF, 5.61 for OMAD, 1.39 for HASAC, 2.32 for HATD3, 4.32 for MAMFPO, and 4.91 for MAFlowRL. For visualization, these values are normalized by the OMAF value of 9.51, resulting in Sample Efficiency scores of 100.00, 59.00, 14.62, 24.40, 45.43, and 51.63 for OMAF, OMAD, HASAC, HATD3, MAMFPO, and MAFlowRL, respectively. A higher score indicates that the target performance is reached with fewer environment interactions.

\paragraph{Exploration} The Exploration dimension measures the coverage of the visited state space during the early stage of training. Specifically, we collect the states encountered by each algorithm during the first 250K training steps on the Ant 2$\times$4 environment and construct a two-dimensional state occupancy map using the 15th and 25th state dimensions. We discretize the state space into a fixed grid with a bin width of 0.5, where the first dimension ranges from $-17.5$ to $17.5$ with 70 bins and the second dimension ranges from $-4$ to $8.5$ with 25 bins, resulting in 1,750 bins in total. Samples falling outside these ranges are discarded rather than assigned to boundary bins. A bin is considered visited if it contains at least one collected state. The resulting numbers of occupied bins are 942 for HATD3, 826 for HASAC, 819 for MAMFPO, 861 for MAFlowRL, 958 for OMAD, and 1,162 for OMAF.

For visualization, we apply min-max normalization across the six algorithms, using 819 as the minimum and 1,162 as the maximum $S_m = \frac{N_m}{1162}\times 100$,
where $N_m$ denotes the number of occupied bins for algorithm $m$. This yields Exploration scores of 100.00, 82.44, 81.07, 74.10, 71.08, and 70.48 for OMAF, OMAD, HATD3, MAFlowRL, HASAC, and MAMFPO, respectively. A higher score indicates broader coverage of the evaluated state space during the first 250K training steps.

\paragraph{Scalability} The Scalability dimension measures the performance of each algorithm under different numbers of agents. We use the Cooperative Navigation environment with $N=10$ and $N=20$ agents and compute the metric from the mean episodic returns at each scale. The reported standard deviations are not used in the calculation. For consistency with the naming of other baselines, IMFPO and IFlowRL are denoted as MAMFPO and MAFlowRL, respectively.

For each agent scale independently, we apply min-max normalization to the mean returns: $S_{N,m} =
\frac{R_{N,m}-R_{N,\min}}
{R_{N,\max}-R_{N,\min}}
\times 100,$
where $R_{N,m}$ is the mean return of algorithm $m$ at agent scale $N$. Since higher returns indicate better performance, a larger normalized score corresponds to better performance at the given scale. At $N=10$, the best and worst mean returns are $-429.4$ (OMAF) and $-478.4$ (MAMFPO), respectively. At $N=20$, they are $-1433.4$ (OMAF) and $-1801.4$ (OMAD), respectively.

We then average the normalized scores across the two agent scales to obtain the final Scalability score. For example, HATD3 obtains a score of $28.78$ at $N=10$ and $96.82$ at $N=20$, resulting in a final score of $62.80$. The resulting Scalability scores are 100.00 for OMAF, 62.80 for HATD3, 54.70 for HASAC, 46.86 for MAFlowRL, 33.98 for OMAD, and 16.40 for MAMFPO. A higher score indicates stronger performance across the evaluated agent scales.

\paragraph{GPU Hours} The GPU Hours dimension measures the computational cost required for 100K training steps on the Ant 2$\times$4 environment, where a higher score indicates lower computational cost. We use the measured training time in seconds per 100K steps and account for the total GPU compute capacity used by each algorithm. Evaluation time is not included in the calculation. IMFPO and IFlowRL are denoted as MAMFPO and MAFlowRL, respectively.

In our experiments, HASAC and HATD3 are trained on 8 NVIDIA 3090 GPUs. OMAF uses 8 NVIDIA 3090 GPUs and 4 NVIDIA A6000 GPUs, while MAMFPO and MAFlowRL use 4 NVIDIA 3090 GPUs and 8 NVIDIA A6000 GPUs. OMAD uses 600 NVIDIA V100 GPUs, 20 NVIDIA H20 GPUs, and 60 NVIDIA A100 GPUs. We estimate the computational cost as the aggregate FP32 peak throughput of the GPUs used by each algorithm multiplied by its wall-clock training time.

Importantly, OMAF achieves its strong performance with dramatically lower computational resource requirements than the diffusion-based OMAD, using only 12 GPUs compared with 680 GPUs in the reported training setup. This substantial reduction in GPU consumption highlights the exceptional training efficiency of OMAF, demonstrating that competitive or stronger policy performance can be achieved with only a small fraction of the computational resources required by iterative diffusion-based policy learning.

Because the resulting computational costs span a large range across methods, we first apply a logarithmic transformation and then linearly map the values to a 0--100 scale. The resulting GPU Hours scores are 100.00 for HASAC, 99.57 for HATD3, 90.00 for OMAF, 89.19 for MAMFPO, 87.66 for MAFlowRL, and 79.45 for OMAD. A higher score indicates lower estimated computational cost per 100K training steps.

\subsubsection{Discussion about the Discrete Action Space}

Our work focuses on efficient online multi-agent coordination in continuous control settings. Accordingly, we evaluate the proposed framework on representative continuous-control benchmarks, including MPE and MAMuJoCo, where agents interact through continuous action spaces and exhibit challenging coordination dynamics arising from physical coupling. We do not consider discrete-action benchmarks such as SMAC in this work, as the proposed flow-based policy is specifically designed to model continuous action distributions. Extending the framework to discrete action spaces, for example through discrete flow formulations or categorical action representations, would require additional methodological development and is left for future work.

\subsection{Visualization}

To complement the quantitative evaluation, we visualize the behaviors learned by OMAF across four representative MAMuJoCo tasks in Figure~\ref{fig:visualizationenv}. For each task, we show the agent configurations at timesteps $t\in{1,100,250,500}$ together with their instantaneous velocities. Across different environments, the agents progressively develop consistent locomotion patterns and maintain stable motion throughout the trajectories. In the CoupledHalfCheetah, Walker2d $2\times3$, and Ant $2\times4$ tasks, the agents exhibit coordinated movements despite their interactions being governed by the underlying coupled dynamics, while the Swimmer $2\times1$ agents maintain a coherent forward-moving pattern. The corresponding episode returns further indicate that these learned behaviors translate into effective control performance, with returns of $8320.8$, $5971.7$, $7495.1$, and $138.9$, respectively. These visualizations provide qualitative evidence that OMAF can learn stable and coordinated locomotion behaviors across diverse multi-agent continuous-control settings.

To further examine the coordination efficiency of OMAF under a limited training budget, we visualize the evaluation trajectories of all six algorithms on the Coupled HalfCheetah environment after a fixed number of training steps. Each trajectory corresponds to one evaluation episode generated by the learned policy. The visualization provides a qualitative comparison of the resulting coordination behaviors and shows that OMAF achieves effective coordination while reaching strong task performance within the given training budget.

\begin{figure}[ht]
    \centering
    \begin{subfigure}{0.9\textwidth}
    \includegraphics[width=\linewidth]{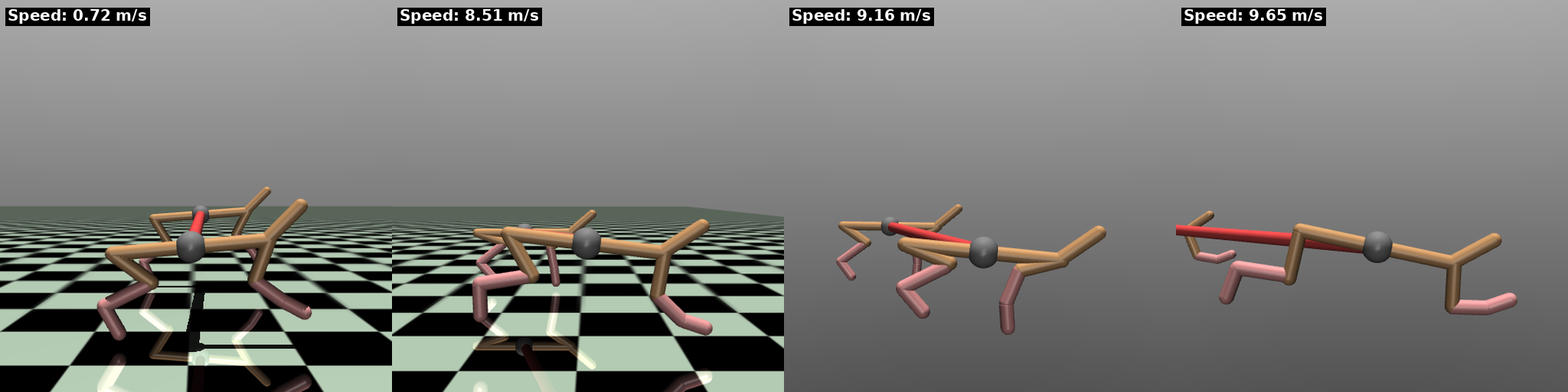}
    \caption{CoupledHalfCheetah Return $=8320.8$}
    \end{subfigure}
    \begin{subfigure}{0.9\textwidth}
    \includegraphics[width=\linewidth]{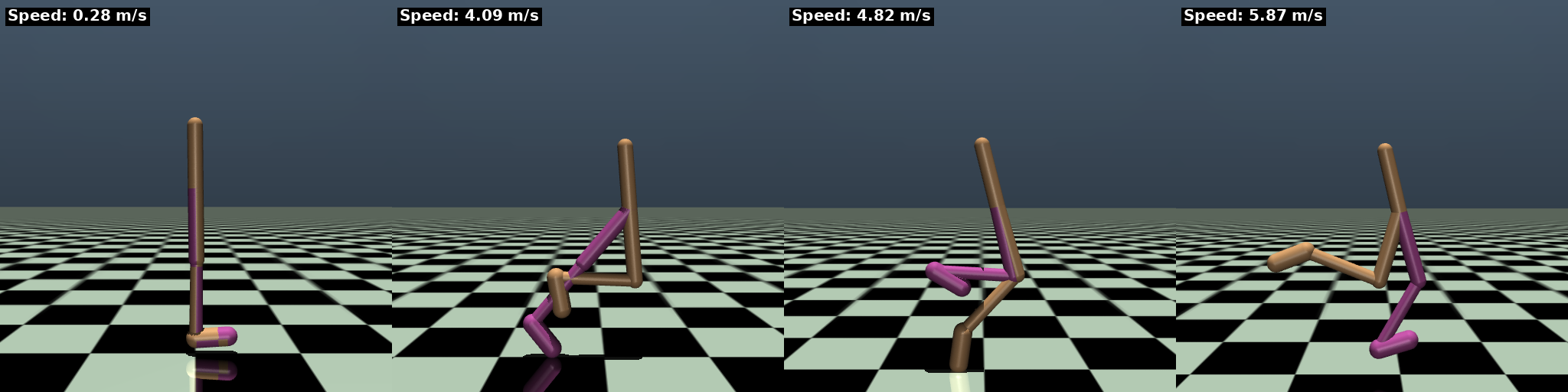}
    \caption{Walker2d 2$\times$3 Return $=5971.7$}
    \end{subfigure}
    \begin{subfigure}{0.9\textwidth}
    \includegraphics[width=\linewidth]{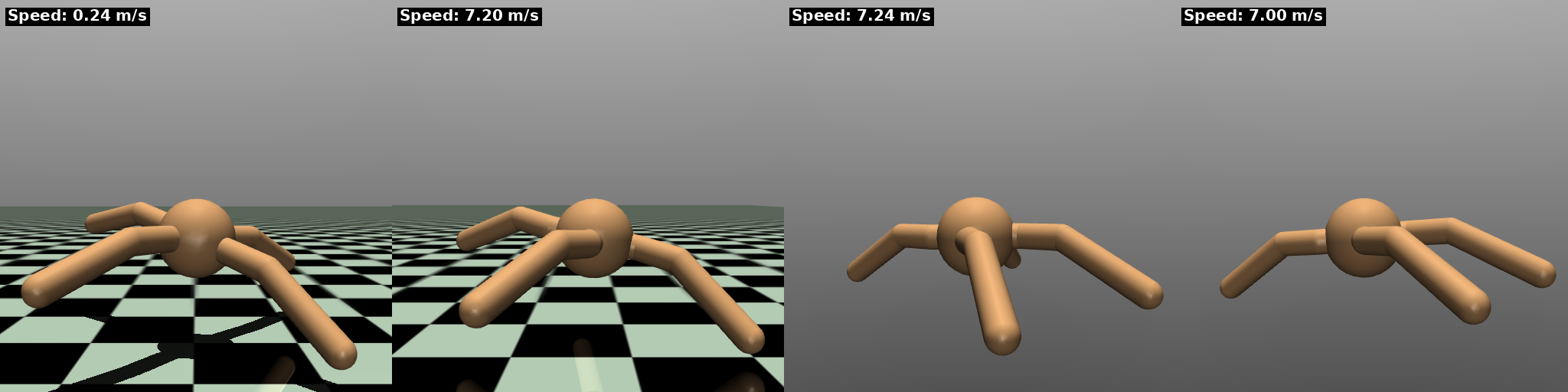}
    \caption{Ant 2$\times$4 Return $=7495.1$}
    \end{subfigure}
    \begin{subfigure}{0.9\textwidth}
    \includegraphics[width=\linewidth]{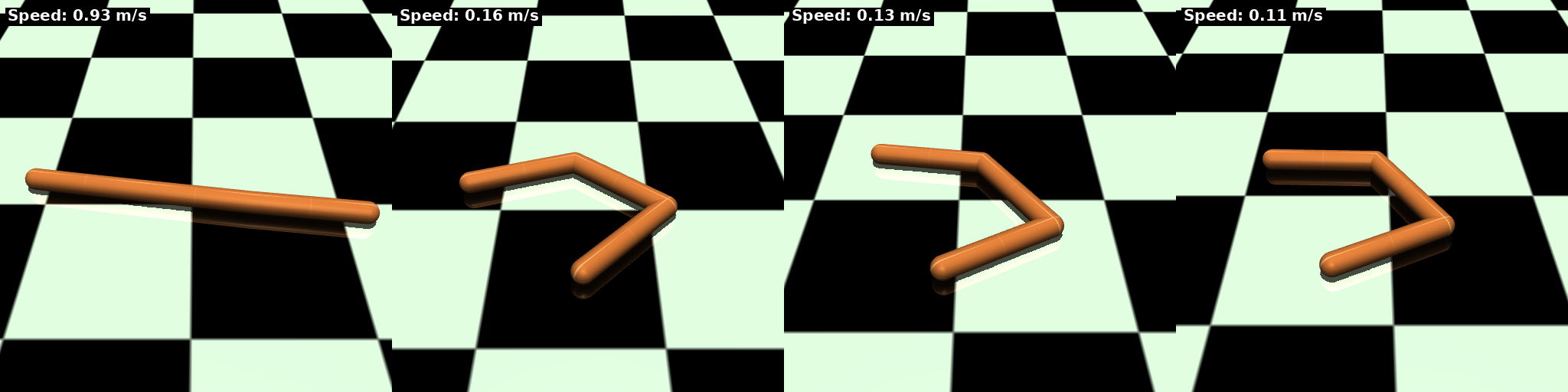}
    \caption{Swimmer 2$\times$1 Return $=138.9$}
    \end{subfigure}
    \caption{Visualization of learned diffusion policies across four distinct MAMuJoCo tasks. We display snapshots of the agents at timesteps $t \in  \{1, 100, 250, 500\}$ with the instantaneous velocity, demonstrating the stable and coordinated behaviors achieved by our OMAF algorithm.}
    \label{fig:visualizationenv}

\end{figure}

\begin{figure}[htbp]
    \centering
    \begin{subfigure}{0.7\textwidth}
    \includegraphics[width=\linewidth]{figure/CoupledHalfCheetah_strip_omaf.png}
    \caption{OMAF Return $=8320.8$}
    \end{subfigure}
    \begin{subfigure}{0.7\textwidth}
    \includegraphics[width=\linewidth]{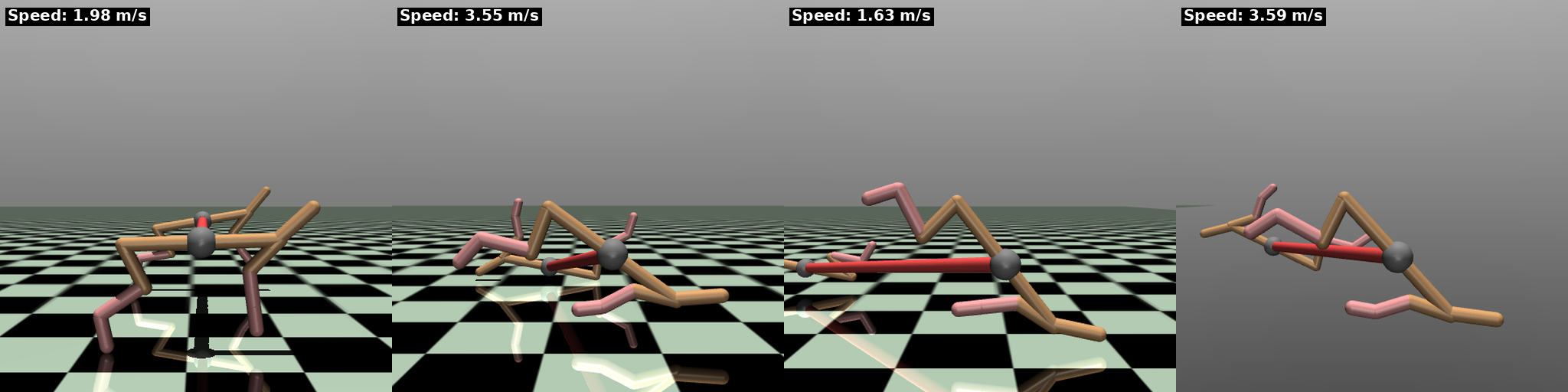}
    \caption{OMAD Return $=1742.09$}
    \end{subfigure}
    \begin{subfigure}{0.7\textwidth}
    \includegraphics[width=\linewidth]{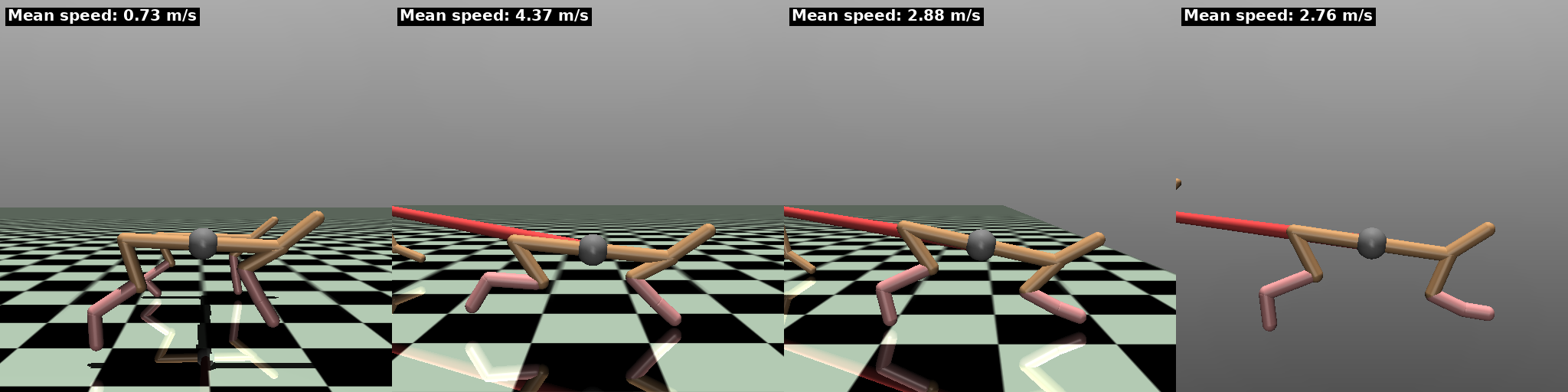}
    \caption{HASAC Return $=2505.96$}
    \end{subfigure}
    \begin{subfigure}{0.7\textwidth}
    \includegraphics[width=\linewidth]{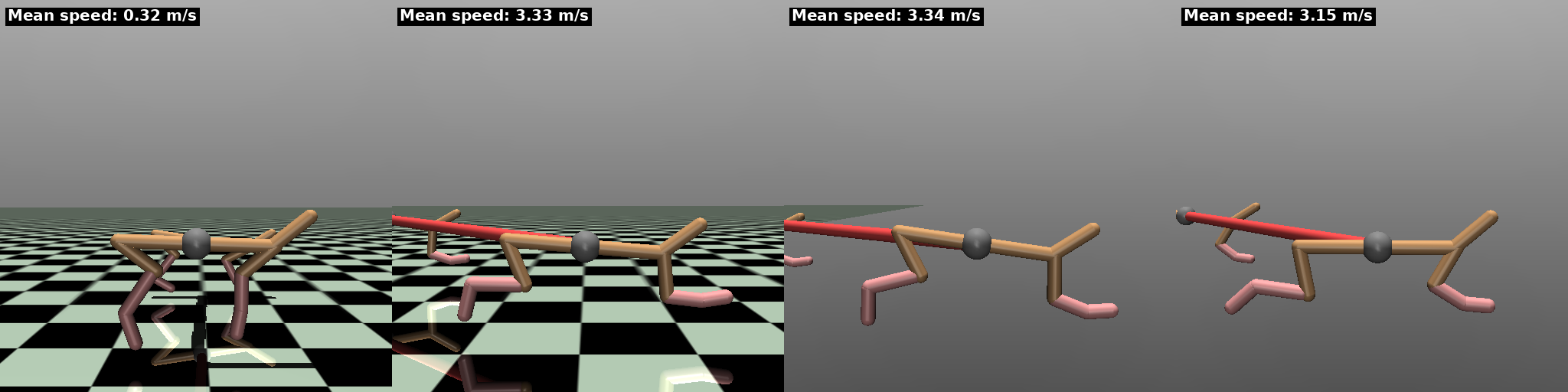}
    \caption{HATD3 Return $=3356.11$}
    \end{subfigure}
    \begin{subfigure}{0.7\textwidth}
    \includegraphics[width=\linewidth]{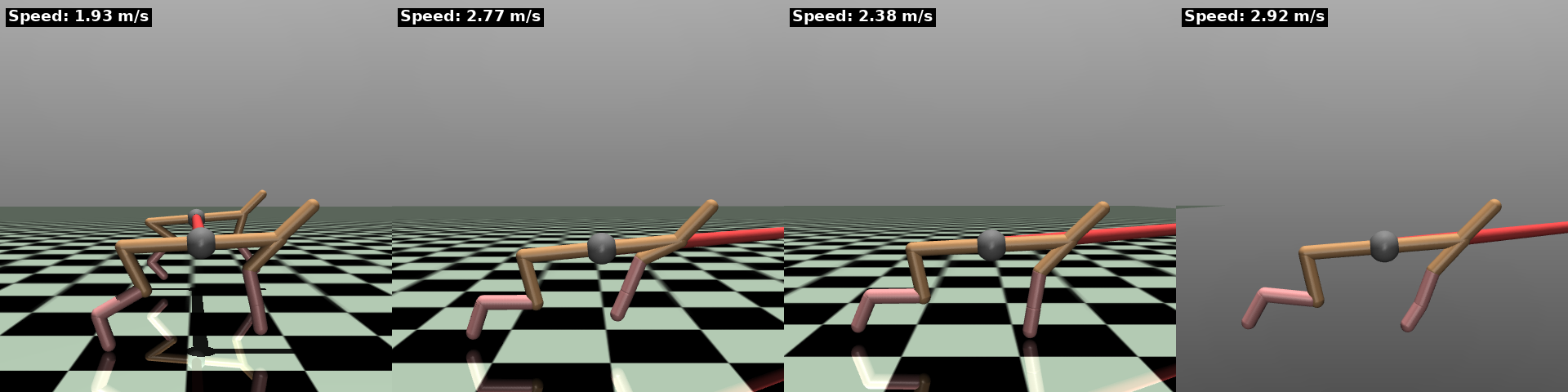}
    \caption{MAMFPO Return $=1841.64$}
    \end{subfigure}
    \begin{subfigure}{0.7\textwidth}
    \includegraphics[width=\linewidth]{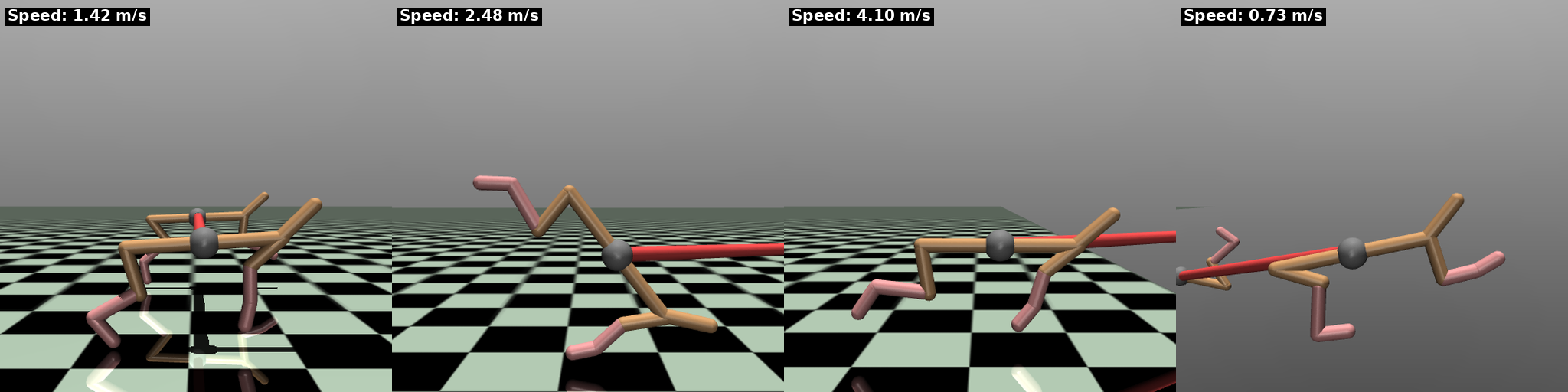}
    \caption{MAFlowRL Return $=2460.30$}
    \end{subfigure}
    \caption{Visualization of OMAF and other baseline algorithms in MAMuJoCo task CoupledHalfCheetah. We display snapshots of the agents at timesteps $t \in  \{1, 100, 250, 500\}$ with the instantaneous velocity, demonstrating the stable and coordinated behaviors achieved by OMAF.}
    \label{fig:visualization-alg}

\end{figure}

\end{document}